\documentclass[lettersize,journal]{IEEEtran}
\usepackage{amsmath,amsfonts}\allowdisplaybreaks
\usepackage{algorithmic}
\usepackage{algorithm}
\usepackage{array}
\usepackage[caption=false,font=normalsize,labelfont=sf,textfont=sf]{subfig}
\usepackage{textcomp}
\usepackage{stfloats}
\usepackage{url}
\usepackage{verbatim}
\usepackage{graphicx}
\usepackage{cite}

\usepackage[american]{babel}
\usepackage{amssymb}
\usepackage{amsthm}
\usepackage{mathtools}
\usepackage{bm}

\newtheorem{definition}{Definition}
\newtheorem{theorem}{Theorem}
\newtheorem{lemma}{Lemma}

\begin{document}

\title{Can Evolutionary Clustering Have Theoretical Guarantees?}

\author{Chao~Qian,~\IEEEmembership{Senior Member,~IEEE}
\thanks{C. Qian is with the State Key Laboratory for Novel Software Technology, Nanjing University, Nanjing 210023, China (e-mail: qianc@nju.edu.cn). https://ieeexplore.ieee.org/document/10185945
}
}

\markboth{IEEE Transactions on Evolutionary Computation,~Vol.~xx, No.~x, ~2022}%
{Qian: Can Evolutionary Clustering Have Theoretical Guarantees?}

\IEEEpubid{0000--0000/00\$00.00~\copyright~2021 IEEE}

\maketitle

\begin{abstract}
Clustering is a fundamental problem in many areas, which aims to partition a given data set into groups based on some distance measure, such that the data points in the same group are similar while that in different groups are dissimilar. Due to its importance and NP-hardness, a lot of methods have been proposed, among which evolutionary algorithms are a class of popular ones. Evolutionary clustering has found many successful applications, but all the results are empirical, lacking theoretical support. This paper fills this gap by proving that the approximation performance of the GSEMO (a simple multi-objective evolutionary algorithm) for solving four formulations of clustering, i.e., $k$-\emph{t}MM, $k$-center, discrete $k$-median and $k$-means, can be theoretically guaranteed. Furthermore, we consider clustering under fairness, which tries to avoid algorithmic bias, and has recently been an important research topic in machine learning. We prove that for discrete $k$-median clustering under individual fairness, the approximation performance of the GSEMO can be theoretically guaranteed with respect to both the objective function and the fairness constraint.
\end{abstract}


\section{Introduction}

Clustering~\cite{kaufman2009finding} aims to partition a set of data points into clusters (groups) such that the data points in the same cluster are similar while the points in different clusters are less similar. It is a fundamental problem in exploratory data analysis, and has been widely studied in many different areas, e.g., machine learning, data mining, pattern recognition, information retrieval, and bioinformatics.

Computing exact solutions to clustering
problems often turns out to be NP-hard, and many techniques have been proposed to find good approximate solutions~\cite{berkhin2006survey}. Among them, Evolutionary Algorithms (EAs) are a class of popular techniques~\cite{jain1999data,mukhopadhyay2015survey}, which have been applied to solve various formulations of clustering, i.e., to optimize various objective functions (measuring the quality of a partition) under different constraints~\cite{hall1999clustering,handl2007evolutionary,kim2011genetic,garza2017improved,tinos2018nk}. Evolutionary clustering has found many successful applications, e.g., gene expression data analysis~\cite{huang2018bi}, complex network analysis~\cite{gong2013complex}, and patient attendance data analysis~\cite{liu2018novel}.

Though achieving good performance in real-world applications, all the results of evolutionary clustering are empirical. To the best of our knowledge, there has been no theoretical analysis. A natural question is then whether we can provide performance guarantees for evolutionary clustering through a theoretical analysis? That is, can we prove that for any problem instance, the objective function value of the clustering generated by EAs is upper bounded (consider minimization) by $\alpha$ times that of the optimal clustering, where $\alpha$ is called approximation ratio? Thus, achieving an approximation guarantee implies that the performance of the algorithm can be guaranteed even in the worst case.

In this paper, we give a positive answer by proving the approximation ratios of EAs for solving four formulations of clustering, i.e., $k$-\emph{t}MM, $k$-center, discrete $k$-median, and $k$-means. Specifically, we consider a simple multi-objective EA (MOEA), i.e., the GSEMO, which employs the mutation operator only and maintains non-dominated solutions generated so far, and has been widely used in theoretical analysis of MOEAs~\cite{Laumanns04,neumann2006minimum,friedrich2010approximating,doerr2013lower,qian-ppsn16-hyper,bian2018tools,qian2019maximizing}.

\IEEEpubidadjcol

Given a set of $n$ points $\mathcal{D}=\{\bm{v}_1,\bm{v}_2,\ldots,\bm{v}_n\}$ in $\mathbb{R}^{l}$, and a metric distance function $d$, the $k$-\emph{t}MM clustering problem~\cite{gonzalez1985clustering} is to find a partition $S_1,S_2,\ldots,S_k$ of $\mathcal{D}$ such that $\max_{m \in \{1,2,\ldots,k\}} \max_{\bm{v}_i,\bm{v}_j \in S_m} d(\bm{v}_i,\bm{v}_j)$ is minimized. That is, the maximum intracluster distance is to be minimized, where the intracluster distance is measured by the largest distance between two points in the cluster. We prove that the GSEMO can achieve a 2-approximation ratio after running at most $ek^2n-ek(n-1)$ expected number of iterations. This approximation guarantee also holds for the $k$-center clustering problem~\cite{har2011geometric}, which aims to find a subset $X\subseteq \mathcal{D}$ of size $k$ such that the maximum distance of a point in $\mathcal{D}$ to the closest point in $X$ is minimized.

For discrete $k$-median clustering~\cite{arya2004local}, we are given another set $F$ of points, and the goal is to select a subset $X \subseteq F$ of size $k$ such that $\sum_{\bm{v}_i \in \mathcal{D}} \min_{\bm{u} \in X} d(\bm{v}_i,\bm{u})$ is minimized. That is, each data point in $\mathcal{D}$ is assigned to the closest point in $X$, forming $k$ clusters, and the sum of the $k$ intracluster distances is to be minimized, where the intracluster distance is measured by the sum of the distance between each point in the cluster and the corresponding median point in $X$. We prove that the GSEMO can achieve a $(3+2/p)/(1-\epsilon)$-approximation ratio, where $p\geq 1$ and $\epsilon>0$. The required expected number of iterations is polynomial in $|F|^p$, $1/\epsilon$, $\log n$ and $\log (\max_i d^{\max}_i/\min_i d^{\min}_i)$, where $d^{\max}_i$ and $d^{\min}_i$ denote the distance between $\bm{v}_i \in \mathcal{D}$ and its farthest and closest (excluding itself) points in $F$, respectively.

The $k$-means clustering problem~\cite{kanungoa2004local} is similar to discrete $k$-median clustering, except that the points to be selected can be arbitrary, and the squared Euclidean distance (which is not a metric) is used. That is, it is to select a set $X \subseteq \mathbb{R}^{l}$ of $k$ points such that $\sum_{\bm{v}_i \in \mathcal{D}} \min_{\bm{u} \in X} \|\bm{v}_i-\bm{u}\|^2$ is minimized. We prove that the GSEMO can achieve a $(1+\epsilon)(3+2/p)^2/(1-\epsilon)^2$-approximation ratio. The required expected number of iterations is polynomial in $n^p$, $1/\epsilon^{lp}$ and $\log (\max_i d^{\max}_i/\min_i d^{\min}_i)$, where $d^{\max}_i$ and $d^{\min}_i$ denote the squared Euclidean distance between $\bm{v}_i \in \mathcal{D}$ and its farthest and closest (excluding itself) points in an $\epsilon$-approximate centroid set of size $O(n\epsilon^{-l}\log(1/\epsilon))$~\cite{matouvsek2000approximate}, respectively.

Furthermore, we consider clustering under fairness, which has attracted much attention recently as optimizing only performance measures may lead to biased outputs of algorithms. In particular, we study discrete $k$-median clustering under individual fairness~\cite{mahabadi2020individual}, where a clustering characterized by a subset $X \subseteq F=\mathcal{D}$ of size $k$ is $\beta$-fair if $\forall \bm{v}_i\in \mathcal{D}: \min_{\bm{u} \in X} d(\bm{v}_i,\bm{u}) \leq \beta\cdot r(\bm{v}_i)$, where $r(\bm{v_i})$ denotes the minimum radius such that the ball centered at $\bm{v_i}$ contains at least $n/k$ points from $\mathcal{D}$. We prove that the GSEMO can achieve a $(84,7)$-bicriteria approximation ratio, where $84$ is the approximation with respect to the optimal objective function value, and $7$ is the approximation with respect to the fairness constraint, implying $\forall \bm{v}_i\in \mathcal{D}: \min_{\bm{u} \in X} d(\bm{v}_i,\bm{u}) \leq 7\beta\cdot r(\bm{v}_i)$. The required expected number of iterations is polynomial in $n$ and $\log (\max_i d^{\max}_i/\min_i d^{\min}_i)$, where $d^{\max}_i$ and $d^{\min}_i$ denote the distance between $\bm{v}_i \in \mathcal{D}$ and its farthest and closest (excluding itself) points in $\mathcal{D}$, respectively.

Note that having theoretical guarantees is not an inherent nature of algorithms. Let's take Lloyd’s algorithm (often called $k$-means algorithm)~\cite{lloyd1982least} as an example, which is probably the most celebrated heuristic for $k$-means clustering, and has been regarded as one of the top 10 algorithms in data mining~\cite{wu2008top}. Starting from $k$ arbitrary centers, it assigns each point in $\mathcal{D}$ to the nearest center, and recomputes each center as the centroid of those points assigned to it; this process is repeated until convergence. Though very appealing in practice due to its simplicity and efficiency, Lloyd’s algorithm guarantees only a local optimum, which can be arbitrarily bad~\cite{kanungoa2004local}. Thus, many efforts have been devoted to improving it (e.g., by using a careful seeding technique~\cite{vassilvitskii2006k} or combining it with a local search strategy~\cite{lattanzi2019better}), to provide performance guarantees through a theoretical analysis.

The contribution of this work is to provide theoretical justification for evolutionary clustering, rather than to show that EAs can achieve the best approximation guarantees. In fact, the 2-approximation ratio achieved by the GSEMO for $k$-\emph{t}MM and $k$-center clustering has been shown to be optimal unless P = NP~\cite{gonzalez1985clustering}, while better approximation ratios can be achieved for the other considered clustering problems by developing other techniques, e.g., $(2.675+\epsilon)$-approximation ratio for discrete $k$-median clustering by dependent rounding~\cite{byrka2017improved}, $(6.357+\epsilon)$-approximation ratio for $k$-means clustering by a primal-dual approach~\cite{ahmadian2019better}, and $(7.081+\epsilon,3)$-bicriteria approximation ratio for discrete $k$-median clustering under individual fairness by a reduction to facility location under matroid constraint~\cite{vakilian2022improved}. Thus, an interesting future work is to improve the approximation ratios of EAs by designing advanced strategies, especially considering that the GSEMO studied in this work is a very basic MOEA, which shares a common evolutionary structure but only uses uniform parent selection, bit-wise mutation, and survivor selection that simply keeps all non-dominated solutions generated so far. The approximation ratios of the GSEMO that are presented in this paper are proved by simulating the behaviors of existing greedy algorithms or local search~\cite{gonzalez1985clustering,arya2004local,kanungoa2004local,mahabadi2020individual}.

The rest of the paper is organized as follows. Section~\ref{sec-prel} first introduces four formulations of the $k$-clustering problem, i.e., $k$-\emph{t}MM, $k$-center, discrete $k$-median, and $k$-means, as well as the GSEMO. In the following, we theoretically analyze the approximation ratios of the GSEMO for $k$-\emph{t}MM, $k$-center, discrete $k$-median, and $k$-means, respectively. Section~\ref{sec-fairness} then presents theoretical analysis of the GSEMO for discrete $k$-median under fairness. Section~\ref{sec-conclusion} concludes this paper.

\section{Preliminaries}\label{sec-prel}

In this section, we first give the formal definitions of $k$-\emph{t}MM, $k$-center, discrete $k$-median, and $k$-means clustering, and then introduce the GSEMO in detail.

\subsection{$k$-Clustering}

Let $\mathbb{R}$ and $\mathbb{R}^+$ denote the set of reals and non-negative reals, respectively. Given a set of $n$ data points $\mathcal{D}=\{\bm{v}_1,\bm{v}_2,\ldots,\bm{v}_n\}$ in $l$-dimensional space $\mathbb{R}^{l}$, the $k$-clustering problem aims to partition $\mathcal{D}$ into $k$ disjoint groups (each group is called a cluster) such that the points in the same group are similar while that in different groups are dissimilar. A typical way to solve this problem is to formulate an objective function measuring the goodness of a partition and then employ optimization techniques. We next introduce four formulations of $k$-clustering, i.e., $k$-\emph{t}MM, $k$-center, discrete $k$-median and $k$-means.

Let $[k]$ denote the set $\{1,2,\ldots,k\}$. As presented in Definition~\ref{def-center}, the objective function of $k$-\emph{t}MM clustering to be minimized is defined as the maximum intracluster distance, where the intracluster distance of a cluster $S_m$ is measured by $\max_{\bm{v}_i,\bm{v}_j \in S_m} d(\bm{v}_i,\bm{v}_j)$, i.e., the maximum distance between two points in $S_m$. Note that the distance function $d: \mathcal{D} \times \mathcal{D} \rightarrow \mathbb{R}^+$ between two points is required to be a metric, satisfying the symmetric property and triangle inequality.

\begin{definition}[$k$-\emph{t}MM Clustering~\cite{gonzalez1985clustering}]\label{def-center}
Given a set of $n$ data points $\mathcal{D}=\{\bm{v}_1,\bm{v}_2,\ldots,\bm{v}_n\}$ in $\mathbb{R}^{l}$, a metric distance function $d: \mathcal{D} \times \mathcal{D} \rightarrow \mathbb{R}^+$, and an integer $k$, the goal of $k$-tMM clustering is to find a partition $S_1,S_2,\ldots,S_k$ of $\mathcal{D}$ such that the maximum intracluster distance, given by
\begin{align}\label{eq-center}
\max\nolimits_{m \in [k]} \max\nolimits_{\bm{v}_i,\bm{v}_j \in S_m} d(\bm{v}_i,\bm{v}_j),
\end{align}
is minimized.
\end{definition}

As presented in Definition~\ref{def-center-real}, given a subset $X \subseteq \mathcal{D}$ of size $k$, the objective function of $k$-center clustering is calculated by the maximum distance of a point in $\mathcal{D}$ to the closest point in $X$. The goal is to find an $X$ minimizing this objective function. Let $X=\{\bm{u}_1,\bm{u}_2,\ldots,\bm{u}_k\}$. Then, each point $\bm{u}_m$ in $X$ corresponds to a cluster $S_m=\{\bm{v}_i \in \mathcal{D} \mid \bm{u}_m =\arg\min_{\bm{u} \in X} d(\bm{v}_i,\bm{u})\}$. That is, each point in $\mathcal{D}$ is assigned to the closest point in $X$ with ties broken arbitrarily. The objective function of $k$-center clustering to be minimized can still be viewed as the maximum intracluster distance, but the intracluster distance of a cluster $S_m$ is now measured by the maximum distance of a point in $S_m$ to its center $\bm{u}_m$. 

\begin{definition}[$k$-Center Clustering~\cite{har2011geometric}]\label{def-center-real}
Given a set of $n$ data points $\mathcal{D}=\{\bm{v}_1,\bm{v}_2,\ldots,\bm{v}_n\}$ in $\mathbb{R}^{l}$, a metric distance function $d: \mathcal{D} \times \mathcal{D} \rightarrow \mathbb{R}^+$, and an integer $k$, the goal of $k$-center clustering is to find a subset $X \subseteq \mathcal{D}$ of size $k$ such that
\begin{align}\label{eq-center-real}
\max\nolimits_{\bm{v}_i \in \mathcal{D}} d(\bm{v}_i,X)
\end{align}
is minimized, where $d(\bm{v}_i,X)=\min \{d(\bm{v}_i,\bm{u}) \mid \bm{u} \in X\}$ is the distance between $\bm{v}_i$ and its closest point in $X$.
\end{definition}

The discrete $k$-median clustering problem as presented in Definition~\ref{def-median} is also called facility location~\cite{arya2004local}. Each data point in $\mathcal{D}$ corresponds to a client, and we are given another set $F$ of points (called facilities). The goal is to open a subset $X \subseteq F$ of $k$ facilities such that the total service cost is minimized, where each client in $\mathcal{D}$ is served by the nearest facility in $X$ and the corresponding service cost is the distance between them. Let $X=\{\bm{u}_1,\bm{u}_2,\ldots,\bm{u}_k\}$. Thus, each facility $\bm{u}_m$ in $X$ corresponds to a cluster $S_m=\{\bm{v}_i \in \mathcal{D} \mid \bm{u}_m =\arg\min_{\bm{u} \in X} d(\bm{v}_i,\bm{u})\}$. That is, each client in $\mathcal{D}$ is assigned to the closest facility in $X$ with ties broken arbitrarily.

\begin{definition}[Discrete $k$-Median Clustering~\cite{arya2004local}]\label{def-median}
Given a set of $n$ data points $\mathcal{D}=\{\bm{v}_1,\bm{v}_2,\ldots,\bm{v}_n\}$ in $\mathbb{R}^{l}$, a set $F$ of points in $\mathbb{R}^{l}$, a metric distance function $d: \mathcal{D}\cup F \times \mathcal{D}\cup F \rightarrow \mathbb{R}^+$, and an integer $k$, the goal of discrete $k$-median clustering is to find a subset $X \subseteq F$ of size $k$ such that
\begin{align}\label{eq-median}
\sum\nolimits_{\bm{v}_i \in \mathcal{D}} d(\bm{v}_i,X)
\end{align}
is minimized, where $d(\bm{v}_i,X)=\min \{d(\bm{v}_i,\bm{u}) \mid \bm{u} \in X\}$ is the distance between $\bm{v}_i$ and its closest point in $X$.
\end{definition}

As presented in Definition~\ref{def-means}, the goal of $k$-means clustering is to determine a set $X$ of $k$ points in $\mathbb{R}^{l}$, to minimize the sum of the squared Euclidean distance from each point in $\mathcal{D}$ to its closest point in $X$. Compared with discrete $k$-median clustering in Definition~\ref{def-median}, the selected $k$ points are not restricted to be from a given set $F$, and the squared Euclidean distance is used, which does not satisfy the triangle inequality, and is not a metric.

\begin{definition}[$k$-Means Clustering~\cite{kanungoa2004local}]\label{def-means}
Given a set of $n$ data points $\mathcal{D}=\{\bm{v}_1,\bm{v}_2,\ldots,\bm{v}_n\}$ in $\mathbb{R}^{l}$, and an integer $k$, the goal of $k$-means clustering is to find a set $X \subseteq \mathbb{R}^{l}$ of size $k$ such that
\begin{align}\label{eq-means}
 \sum\nolimits_{\bm{v}_i \in \mathcal{D}} \min_{\bm{u} \in X} \|\bm{v}_i-\bm{u}\|^2
\end{align}
is minimized, where $\|\bm{v}_i-\bm{u}\|^2$ is the squared Euclidean distance between $\bm{v}_i$ and $\bm{u}$.
\end{definition}

Solving the above four formulations of $k$-clustering exactly for general $k$ is NP-hard~\cite{gonzalez1985clustering,guha1999greedy,drineas2004clustering}, and thus many efforts have been devoted to developing algorithms with approximation guarantees, e.g.,~\cite{gonzalez1985clustering,arya2004local,kanungoa2004local,byrka2017improved,ahmadian2019better}. Given a minimization problem where the objective function is denoted as $f$, an algorithm is said to achieve an $\alpha$-approximation ratio if for any instance of this problem, the output solution $X$ of this algorithm satisfies $f(X) \leq \alpha \cdot \mathrm{OPT}$, where $\alpha \geq 1$ and $\mathrm{OPT}$ denotes the optimal function value. The approximation ratio of an algorithm implies its approximation performance in the worst case. The smaller the value of $\alpha$, the better worst-case scenario performance of the algorithm.

\subsection{GSEMO}

To examine whether EAs can achieve approximation guarantees for $k$-clustering, we consider the Global Simple Evolutionary Multi-objective Optimizer (GSEMO)~\cite{Laumanns04} as presented in Algorithm~\ref{algo:GSEMO}, which is used for maximizing multiple pseudo-Boolean objective functions over $\{0,1\}^n$ simultaneously. The GSEMO can be viewed as a counterpart of the well-studied (1+1)-EA in theoretical analysis of single-objective EAs~\cite{droste1998rigorous,pourhassan2020runtime,bian2021robustness}. In multi-objective maximization $\max\, (f_1,f_2,\ldots,f_m)$, solutions may be incomparable due to the conflicting of objectives. The domination relationship in Definition~\ref{def_Domination} is often used for comparison.
\begin{definition}[Domination]\label{def_Domination}
For two solutions $\bm x, \bm{x}' \in \{0,1\}^n$,
\begin{enumerate}
  \item $\bm{x}$ \emph{weakly dominates} $\bm{x}'$ (i.e., $\bm{x}$ is \emph{better} than $\bm{x}'$, denoted by $\bm{x} \succeq \bm{x}'$) if \;$\forall i\in [m]: f_i(\bm{x}) \geq f_i(\bm{x}')$;
  \item ${\bm{x}}$ \emph{dominates} $\bm{x}'$ (i.e., $\bm{x}$ is \emph{strictly better} than $\bm{x}'$, denoted by $\bm{x} \succ \bm{x}'$) if ${\bm{x}} \succeq \bm{x}' \wedge \exists i\in [m]: f_i(\bm{x}) > f_i(\bm{x}')$.
\end{enumerate}
\end{definition}
\noindent Two solutions $\bm{x}$ and $\bm{x}'$ are \emph{incomparable} if neither $\bm{x} \succeq \bm{x}'$ nor $\bm{x}' \succeq \bm{x}$. A solution is \emph{Pareto optimal} if no other solution dominates it. The \emph{Pareto set} consists of all Pareto optimal solutions, i.e., it is the set of non-dominated solutions. The collection of objective vectors of all Pareto optimal solutions is called \emph{Pareto front}.

\begin{algorithm}[t!]\caption{GSEMO for $k$-Clustering}\label{algo:GSEMO}
\textbf{Input}: $m$ pseudo-Boolean functions $f_1,f_2,\ldots,f_m$, where $f_i: \{0,1\}^n \rightarrow \mathbb{R}$\\
\textbf{Process}:
    \begin{algorithmic}[1]
    \STATE Let $P \gets \{\bm 0\}$;
    \STATE \textbf{repeat}
    \STATE \quad Choose $\bm x$ from $P$ uniformly at random;
    \STATE \quad Create $\bm{x}'$ by flipping each bit of $\bm x$ with prob. $1/n$;
    \STATE \quad \textbf{if} \, {$|\bm{x}'|\leq k$ and $\nexists \bm z \in P$ such that $\bm z \succ \bm {x}'$} \,\textbf{then}
    \STATE  \qquad $P \gets (P \setminus \{\bm z \in P \mid \bm {x}' \succeq \bm z\}) \cup \{\bm {x}'\}$
    \STATE \quad \textbf{end if}
    \STATE \textbf{until} some criterion is met
    \end{algorithmic}
\end{algorithm}

The GSEMO starts from the all-0s vector $\bm{0}$ in line~1, and iteratively tries to improve the quality of solutions in the population $P$ (i.e., lines~2--8). In each iteration, a parent solution $\bm{x}$ is selected from $P$ uniformly at random in line~3, and used to generate an offspring solution $\bm{x}'$ by bit-wise mutation in line~4, which flips each bit of $\bm{x}$ independently with probability $1/n$. The newly generated offspring solution $\bm{x}'$ is then used to update the population $P$ (lines~5--7). Let $|\bm{x}'|=\sum^n_{i=1}x'_i$ denote the number of 1-bits contained by $\bm{x}'$. If $\bm{x}'$ contains at most $k$ 1-bits (i.e., $|\bm{x}'|\leq k$) and it is not dominated by any solution in $P$, it will be added into $P$, and those solutions (i.e., $\{\bm z \in P \mid \bm {x}' \succeq \bm z\}$) weakly dominated by $\bm{x}'$ will be deleted in line~6. By this updating procedure, the population $P$ will maintain the solutions with size at most $k$ which correspond to all the non-dominated objective vectors  generated so far. Note that for those non-dominated solutions with the same objective vector, only the latest one is kept in $P$. For ease of theoretical analysis (e.g., in Section~\ref{sec-center}), we have slightly modified the original version of the GSEMO in~\cite{Laumanns04}, by using the all-0s vector $\bm{0}$ as the initial solution instead of sampling it from $\{0,1\}^n$ uniformly at random, and deleting the solutions with size larger than $k$ directly.

As EAs are general-purpose algorithms, we only consider the GSEMO in the analysis for different variants of $k$-clustering, to reflect this property. To apply the GSEMO to solve the $k$-clustering problem, one needs to first decide a way of solution representation, i.e., how to use a Boolean vector to represent a solution of $k$-\emph{t}MM, $k$-center, discrete $k$-median or $k$-means; then reformulate the $k$-clustering problem as a bi-objective maximization problem
\begin{align}\label{eq-bi-problem}
&\max\nolimits_{\bm{x} \in \{0,1\}^n} \;\;  (f_1(\bm{x}),f_2(\bm{x}));
\end{align}
finally run the GSEMO and select a solution from the population to output when terminated. Rephrasing a single-objective problem in a multi-objective way is a useful technique~\cite{knowles2001reducing}, whose effectiveness has been proved recently for EAs solving several combinatorial optimization problems, e.g., minimum spanning tree~\cite{neumann2006minimum}, covering~\cite{friedrich2010approximating}, minimum cuts~\cite{neumann2011computing}, minimum cost coverage~\cite{qian.ijcai15}, submodular optimization~\cite{friedrich2015maximizing,qian2019maximizing}, and result diversification~\cite{qian2022result}. 

In the next three sections, we will show how to implement this procedure for $k$-\emph{t}MM, $k$-center, discrete $k$-median and $k$-means, respectively, and analyze the expected number of iterations of the GSEMO required to reach some approximation ratios for the first time. Note that our focus is the quality of the final output solution with respect to the original $k$-clustering problem, rather than the approximation of the final population to the Pareto front of the reformulated bi-objective problem. In our implementations, the number of 1-bits of a Boolean-vector solution always corresponds to the number of clusters; thus, the solutions with more than $k$ clusters are excluded during the optimization process of the GSEMO. In the final population of the GSEMO, a solution with $k$ 1-bits (i.e., $k$ clusters) will be output as the generated solution. 

\section{Theoretical Analysis of The GSEMO\\ for $k$-\emph{t}MM and $k$-Center Clustering}\label{sec-center}

To apply the GSEMO to solve the $k$-\emph{t}MM clustering problem in Definition~\ref{def-center}, we use a Boolean vector $\bm{x} \in \{0,1\}^n$ to represent a partition $\{S_1,S_2,\ldots,S_{|\bm{x}|}\}$ of $\mathcal{D}$, where $|\mathcal{D}|=n$, and $|\bm{x}|=\sum^n_{i=1}x_i$ denotes the size of $\bm{x}$. If the $i$-th bit $x_i=1$, the data point $\bm{v}_i \in \mathcal{D}$ serves as the center of a cluster; thus there are $|\bm{x}|$ centers (also clusters) in total. Each point in $\mathcal{D}$ is assigned to the closest center, forming the $|\bm{x}|$ clusters $S_1,S_2,\ldots,S_{|\bm{x}|}$. Note that by this way of solution representation, we are actually searching in a strict subspace of all possible partitions of $\mathcal{D}$, which is, however, sufficient to guarantee a good approximation ratio, as shown in Theorem~\ref{theo-center}.

We use $X=\{\bm{v}^*_1,\bm{v}^*_2,\ldots,\bm{v}^*_{|\bm{x}|}\}$ to denote the $|\bm{x}|$ points corresponding to those 1-bits in $\bm{x}$, and $\bm{v}^*_i$ serves as the center of $S_i$, i.e., $S_i=\{\bm{v} \in \mathcal{D} \mid \bm{v}^*_i =\arg\min_{\bm{u} \in X} d(\bm{v},\bm{u})\}$. After determining the way of solution representation, the original problem in Definition~\ref{def-center} is reformulated as a bi-objective maximization problem
\begin{align}\label{eq-bi-center}
&\max\nolimits_{\bm{x} \in \{0,1\}^n} \;\; (f_1(\bm{x}),f_2(\bm{x})),\\
&\text{where}\;\begin{cases}\nonumber
f_1(\bm{x}) = \min\limits_{\bm{v}^*_{i},\bm{v}^*_{j} \in X} d(\bm{v}^*_{i},\bm{v}^*_{j})-\max\limits_{m \in [|\bm{x}|]} \max\limits_{\bm{v} \in S_m} d(\bm{v},\bm{v}^*_m),\\
f_2(\bm x) = |\bm{x}|.
\end{cases}
\end{align}
In the definition of $f_1$, $\min\nolimits_{\bm{v}^*_{i},\bm{v}^*_{j} \in X} d(\bm{v}^*_{i},\bm{v}^*_{j})$ is the minimum distance between cluster centers, characterizing the dissimilarity between clusters, while $\max\nolimits_{m \in [|\bm{x}|]} \max\nolimits_{\bm{v} \in S_m} d(\bm{v},\bm{v}^*_m)$ is the maximum distance from each point to its center, characterizing the dissimilarity within clusters. Thus, maximizing $f_1$ will prefer high intercluster distances and low intracluster distances, which is consistent with the goal of clustering. To be well defined, we set the $f_1$ value to $+\infty$ for the all-0s vector $\bm{0}$ and the vectors with size 1. Note that the Boolean-vector solutions with size larger than $k$ are excluded during the optimization process of the GSEMO. When the GSEMO terminates, the population may contain several non-dominated solutions with different sizes, and the Boolean-vector solution with size $k$ (corresponding to a partition $\{S_1,S_2,\ldots,S_k\}$ of $\mathcal{D}$) in the population will be output.

Theorem~\ref{theo-center} shows that the GSEMO achieves a $2$-approximation ratio after running at most $ek^2n-ek(n-1)$ expected number of iterations. This has been shown to be the optimal polynomial-time approximation ratio unless P = NP~\cite{gonzalez1985clustering}, and implies that the partition $\{S_1,S_2,\ldots,S_k\}$ of $\mathcal{D}$ corresponding to the Boolean-vector solution output by the GSEMO satisfies $\max_{m \in [k]} \max_{\bm{v}_i,\bm{v}_j \in S_m} d(\bm{v}_i,\bm{v}_j) \leq 2 \cdot \mathrm{OPT}$, where $\mathrm{OPT}$ denotes the optimal value of Eq.~(\ref{eq-center}). The proof is inspired by the analysis of the 2-approximation algorithm in~\cite{gonzalez1985clustering}, which starts from an arbitrary center, and iteratively selects the point with the maximum distance to its center as a new center, until having $k$ centers. 
 
\begin{theorem}\label{theo-center}
For $k$-tMM clustering in Definition~\ref{def-center}, the expected number of iterations of the GSEMO using Eq.~(\ref{eq-bi-center}), until achieving a $2$-approximation ratio, is at most $ek^2n-ek(n-1)$.
\end{theorem}
\begin{proof}
To analyze the expected number of iterations until achieving a $2$-approximation ratio, we consider
\begin{align*}
&J_{\max}=\max\{|\bm{x}| \mid \bm{x} \in P, f_1(\bm{x}) \geq 0 \}.
\end{align*}
When $J_{\max}=k$, there exists one Boolean-vector solution $\bm{x}$ in the population $P$ satisfying that $|\bm{x}| = k$ and
\begin{align}\label{eq-center-1}
f_1(\bm{x}) = \min\limits_{\bm{v}^*_{i},\bm{v}^*_{j} \in X} d(\bm{v}^*_{i},\bm{v}^*_{j})-\max\limits_{m \in [k]} \max\limits_{\bm{v} \in S_m} d(\bm{v},\bm{v}^*_m)\geq 0.
\end{align}
As noted before, $X=\{\bm{v}^*_1,\bm{v}^*_2,\ldots,\bm{v}^*_{k}\}$ represents the $k$ data points corresponding to those 1-bits in $\bm{x}$; and for each $m\in [k]$, $\bm{v}^*_m$ serves as the center of $S_m$, and $S_m=\{\bm{v} \in \mathcal{D} \mid \bm{v}^*_m =\arg\min_{\bm{u} \in X} d(\bm{v},\bm{u})\}$. Let $\max\nolimits_{m \in [k]} \max\nolimits_{\bm{v} \in S_m} d(\bm{v},\bm{v}^*_m)=h$, and we use $\bm{v}^*$ to denote the point which has the distance $h$ to its center. As the distance function $d$ is a metric, satisfying the triangle inequality, we have that the distance between two points in any cluster is at most $2h$, implying that the objective value, i.e., $\max_{m \in [k]} \max_{\bm{v}_i,\bm{v}_j \in S_m} d(\bm{v}_i,\bm{v}_j)$, of the partition $\{S_1,S_2,\ldots,S_k\}$ represented by $\bm{x}$ is at most $2h$. That is,
\begin{align}\label{eq-center-2}
\max_{m \in [k]} \max_{\bm{v}_i,\bm{v}_j \in S_m} d(\bm{v}_i,\bm{v}_j) \leq 2h.
\end{align}
By Eq.~(\ref{eq-center-1}), we have \begin{align}\label{eq-center-11}\min\limits_{\bm{v}^*_{i},\bm{v}^*_{j} \in X} d(\bm{v}^*_{i},\bm{v}^*_{j}) \geq \max\limits_{m \in [k]} \max\limits_{\bm{v} \in S_m} d(\bm{v},\bm{v}^*_m)=h.\end{align}
Because each point is assigned to the closest center in $X=\{\bm{v}^*_1,\bm{v}^*_2,\ldots,\bm{v}^*_{k}\}$, and $\bm{v}^*$ has distance $h$ to its center, we can conclude that the distance from $\bm{v}^*$ to any center in $X$ is at least $h$, i.e., 
\begin{align}\label{eq-center-12}
\min\nolimits_{\bm{v}^*_{i} \in X} d(\bm{v}^*,\bm{v}^*_{i}) \geq h.
\end{align}
By Eqs.~(\ref{eq-center-11}) and~(\ref{eq-center-12}), we have that the distance between any two points in $X \cup \{\bm{v}^*\}=\{\bm{v}^*_1,\bm{v}^*_2,\ldots,\bm{v}^*_{k},\bm{v}^*\}$ is at least $h$. As $|X \cup \{\bm{v}^*\}|=k+1$, for any partition (leading to $k$ clusters) of $\mathcal{D}$, there must exist one cluster containing at least two points in $X \cup \{\bm{v}^*\}$, implying that the objective function value of the optimal partition is at least $h$. That is,
\begin{align}\label{eq-center-3}
    h \leq \mathrm{OPT}.
\end{align}
Combining Eqs.~(\refeq{eq-center-2}) and~(\refeq{eq-center-3}) leads to
\begin{align}
\max_{m \in [k]} \max_{\bm{v}_i,\bm{v}_j \in S_m} d(\bm{v}_i,\bm{v}_j) \leq 2\cdot \mathrm{OPT},\label{eq-center-14}
\end{align}
that is, the desired approximation guarantee is reached. Let $T$ denote the number of iterations run by the GSEMO. Thus, we only need to analyze the expected number of iterations until $J_{\max}=k$, i.e., $\mathbb{E}[\min\{T \geq 0 \mid J_{\max}=k\; \text{after $T$ iterations}\}]$, where $\mathbb{E}[\cdot]$ denotes the expectation of a random variable. 

It can be easily observed that $J_{\max}$ will not decrease. Let $\bm{x}$ denote the solution in the population $P$, corresponding to $J_{\max}$. That is, $f_1(\bm{x})\geq 0$ and $|\bm{x}|=J_{\max}$. If $\bm{x}$ keeps in $P$, $J_{\max}$ obviously does not decrease. If it is deleted from $P$ in line~6 of Algorithm~\ref{algo:GSEMO}, the newly included offspring solution $\bm{x}'$ must weakly dominate $\bm{x}$, implying that $|\bm{x}'| \geq |\bm{x}|$ and $f_1(\bm{x}')\geq f_1(\bm{x})\geq 0$. Before analyzing the increase of $J_{\max}$, we first derive an upper bound on the population size $P$, which will be frequently used in the following analysis. According to the procedure of updating the population $P$ in lines~5--7 of Algorithm~\ref{algo:GSEMO}, the solutions maintained in $P$ must be incomparable. Since two solutions having the same value on one objective are comparable, the population $P$ contains at most one solution for each value of one objective. As the Boolean-vector solutions with size larger than $k$ are excluded, $f_2(\bm{x})=|\bm{x}|$ can take values $0,1,\ldots,k$, implying $|P| \leq k+1$. To be more precise, $|P| \leq k$, because the all-0s vector $\bm{0}$ (which has $f_1(\bm{0})=+\infty$ and $f_2(\bm{0})=0$) is dominated by any solution $\bm{x}$ with size 1 (which has $f_1(\bm{x})=+\infty$ and $f_2(\bm{x})=1$). Next, we are to analyze the expected number of iterations required to increase $J_{\max}$ under its different values.

Because the GSEMO starts from $\bm{0}$, $J_{\max}= 0$ initially. By selecting the solution $\bm{0}$ in line~3 of Algorithm~1 and flipping only one 0-bit in line~4, an offspring solution $\bm{x}'$ with size 1 (i.e., $|\bm{x}'|=1$) will be generated. Note that $\bm{x}'$ must be added into the population $P$, because no other solution can dominate it. As $f_1(\bm{x}')=+\infty$ and $|\bm{x}'|=1$, such a selection and mutation behavior will make $J_{\max}=1$, implying that $J_{\max}$ increases. The probability of selecting $\bm{0}$ in line~3 is $1/|P| \geq 1/k$ due to uniform selection and $|P| \leq k$. The probability of flipping only one 0-bit of the vector $\bm{0}$ in line~4 is $n \cdot (1/n)(1-1/n)^{n-1}\geq 1/e$. Thus, when $J_{\max}= 0$, $J_{\max}$ increases in one iteration with probability at least $1/(ek)$, i.e., 
$$\mathrm{Pr}[J_{\max} \;\text{increases in one iteration} \mid J_{\max}=0]\geq 1/(ek).$$

When $J_{\max}=1$, let $\bm{x}$ denote the corresponding Boolean-vector solution in the population $P$, which has size 1, i.e., $|\bm{x}|=1$. We use $\bm{v}^*_1 \in \mathcal{D}$ to denote the data point corresponding to the only 1-bit of $\bm{x}$, and use $\bm{v}^*_2 \in \mathcal{D}$ to denote the point having the maximum distance (denoted as $h$) with $\bm{v}^*_1$. By selecting the solution $\bm{x}$ in line~3 of Algorithm~1 and flipping only the 0-bit corresponding to $\bm{v}^*_2$ in line~4, an offspring solution $\bm{x}'$ with size 2 (i.e., $|\bm{x}'|=2$) will be generated, which contains two points, i.e., $\bm{v}^*_1$ and $\bm{v}^*_2$, serving as two centers and leading to two clusters $S_1$ and $S_2$. To examine whether $f_1(\bm{x}')\geq 0$, we first have \begin{align}\label{eq-center-4}
d(\bm{v}^*_{1},\bm{v}^*_{2})=h.
\end{align}
As each point is assigned to the closest center, we have for any $\bm{v} \in S_2$, $d(\bm{v},\bm{v}^*_2) \leq  d(\bm{v},\bm{v}^*_1)$. Thus, 
\begin{align}\label{eq-center-5}
\max\limits_{m \in [2]} \max\limits_{\bm{v} \in S_m} d(\bm{v},\bm{v}^*_m) \leq \max\limits_{\bm{v} \in \mathcal{D}} d(\bm{v},\bm{v}^*_1)=h.
\end{align}
Combining Eqs.~(\refeq{eq-center-4}) and~(\refeq{eq-center-5}) leads to
\begin{align*}
f_1(\bm{x}') = d(\bm{v}^*_{1},\bm{v}^*_{2})-\max\limits_{m \in [2]} \max\limits_{\bm{v} \in S_m} d(\bm{v},\bm{v}^*_m)\geq 0.
\end{align*}
Once $\bm{x}'$ is generated, it will be added into the population $P$; otherwise, $\bm{x}'$ must be dominated by one solution in $P$ (line~5 of Algorithm~\ref{algo:GSEMO}), and this implies that $J_{\max}$ has already been larger than $1$, contradicting $J_{\max}=1$. After including $\bm{x}'$ into $P$, $J_{\max} = 2$, implying that $J_{\max}$ increases. The probability of selecting $\bm{x}$ in line~3 is $1/|P| \geq 1/k$, and the probability of flipping only a specific 0-bit of $\bm{x}$ in line~4 is $(1/n)(1-1/n)^{n-1} \geq 1/(en)$. Thus, when $J_{\max}= 1$, $J_{\max}$ increases in one iteration with probability at least $1/(ekn)$, i.e., 
$$\mathrm{Pr}[J_{\max} \;\text{increases in one iteration} \mid J_{\max}=1]\geq 1/(ekn).$$

When $J_{\max}= i \geq 2$, we also use $\bm{x}$ to denote the corresponding Boolean-vector solution in the population $P$, satisfying that $f_1(\bm{x}) \geq 0$ and $|\bm{x}|=i$. Let $X=\{\bm{v}^*_1,\bm{v}^*_2,\ldots,\bm{v}^*_{i}\}$ represent the $i$ data points corresponding to those 1-bits in $\bm{x}$, and the corresponding partition of $\mathcal{D}$ is $\{S_1,S_2,\ldots,S_i\}$, where $\forall m \in [i]$, $\bm{v}^*_{m}$ serves as the center of $S_m$. Let $\bm{v}^*_{i+1} \in \mathcal{D}$ denote the point having the maximum distance (denoted as $h$) to its center, denoted as $\bm{v}^*_{j}$. That is,
\begin{align}\label{eq-center-6}
d(\bm{v}^*_{i+1},\bm{v}^*_{j})=\max\limits_{m \in [i]} \max\limits_{\bm{v} \in S_m} d(\bm{v},\bm{v}^*_m)=h.
\end{align}
Similar to the analysis of $J_{\max}=1$, by selecting the solution $\bm{x}$ in line~3 of Algorithm~1 and flipping only the 0-bit corresponding to $\bm{v}^*_{i+1}$ in line~4, an offspring solution $\bm{x}'$ with size $i+1$ (i.e., $|\bm{x}'|=i+1$) will be generated, which contains $i+1$ points, i.e., $X'=\{\bm{v}^*_1,\bm{v}^*_2,\ldots,\bm{v}^*_{i+1}\}$, serving as $i+1$ centers and leading to $i+1$ clusters $S'_1,S'_2,\ldots,S'_{i+1}$. Next we will show $f_1(\bm{x}')\geq 0$. By using Eq.~(\ref{eq-center-6}) and
\begin{align*}
f_1(\bm{x}) = \min\limits_{\bm{v}^*_{p},\bm{v}^*_{q} \in X} d(\bm{v}^*_{p},\bm{v}^*_{q})-\max\limits_{m \in [i]} \max\limits_{\bm{v} \in S_m} d(\bm{v},\bm{v}^*_m)\geq 0,
\end{align*}
we have 
\begin{align}\label{eq-center-7}
 \min\nolimits_{\bm{v}^*_{p},\bm{v}^*_{q} \in X} d(\bm{v}^*_{p},\bm{v}^*_{q})\geq h.
\end{align}
As each point is assigned to the closest center, and the center of $\bm{v}^*_{i+1}$ is $\bm{v}^*_{j}$ under the partition $S_1,S_2,\ldots,S_i$, we have
\begin{align}\label{eq-center-8}
    \forall m \in [i]: d(\bm{v}^*_{i+1},\bm{v}^*_{m}) \geq d(\bm{v}^*_{i+1},\bm{v}^*_{j})=h.
\end{align}
Combining Eqs.~(\refeq{eq-center-7}) and~(\refeq{eq-center-8}) leads to
\begin{align}\label{eq-center-9}
 \min\nolimits_{\bm{v}^*_{p},\bm{v}^*_{q} \in X'} d(\bm{v}^*_{p},\bm{v}^*_{q})\geq h,
\end{align}
where $X'=\{\bm{v}^*_1,\bm{v}^*_2,\ldots,\bm{v}^*_{i+1}\}=X\cup \{\bm{v}^*_{i+1}\}$. Again using the fact that each data point is assigned to the closest center, we know that $\forall m \in [i]$, $S'_m \subseteq S_m$, and each point in $S_m \setminus S'_m$ is moved to $S'_{i+1}$ because it is now closer to the new center $\bm{v}^*_{i+1}$ than $\bm{v}^*_{m}$. Thus, we have
\begin{align}\label{eq-center-10}
\max\limits_{m \in [i+1]} \max\limits_{\bm{v} \in S'_m} d(\bm{v},\bm{v}^*_m) \leq \max\limits_{m \in [i]} \max\limits_{\bm{v} \in S_m} d(\bm{v},\bm{v}^*_m)=h.
\end{align}
Combining Eqs.~(\refeq{eq-center-9}) and~(\refeq{eq-center-10}) leads to
\begin{align*}
f_1(\bm{x}') = \min\limits_{\bm{v}^*_{p},\bm{v}^*_{q} \in X'} d(\bm{v}^*_{p},\bm{v}^*_{q})-\max\limits_{m \in [i+1]} \max\limits_{\bm{v} \in S'_m} d(\bm{v},\bm{v}^*_m)\geq 0.
\end{align*}
After generating $\bm{x}'$, which satisfies that $f_1(\bm{x}')\geq 0$ and $|\bm{x}'|=i+1$, and must be added into the population $P$, we have $J_{\max} = i+1$, implying that $J_{\max}$ increases. As analyzed for the case of $J_{\max}=1$, the probability of selecting $\bm{x}$ in line~3 and flipping only a specific 0-bit of $\bm{x}$ in line~4 is at least $(1/k) \cdot (1/(en))$. Thus, when $J_{\max}\geq 2$, $J_{\max}$ increases in one iteration with probability at least $1/(ekn)$, i.e., 
$$\mathrm{Pr}[J_{\max} \;\text{increases in one iteration} \mid J_{\max}\geq 2]\geq 1/(ekn).$$

Combining the above analyses for $J_{\max}=0$, $J_{\max}=1$ and $J_{\max}\geq 2$, we can conclude that the expected number of iterations until $J_{\max}=k$ (i.e., achieving a $2$-approximation ratio) is at most $ek+ekn+ekn(k-2)=ek^2n-ek(n-1)$.
\end{proof}

From the above proof, we see that when we set the initial solution of the GSEMO to the all-0s vector, we are able to make the initial value of $J_{\max}$ be 0. However, from a random starting vector selected from $\{0,1\}^n$, $J_{\max}$ may be not well defined, and it may require a lot of time to find a Boolean-vector solution $\bm{x}$ with $f_1(\bm{x}) \geq 0$. 

When applying the GSEMO to solve the $k$-\emph{t}MM clustering problem, the search space led by the adopted way of solution representation is exactly the whole solution space of the $k$-center clustering problem in Definition~\ref{def-center-real}. In fact, the approximation guarantee of the GSEMO for $k$-\emph{t}MM clustering also holds for $k$-center clustering, because the event $J_{\max}=k$ considered in the proof of Theorem~\ref{theo-center} directly implies a $2$-approximation ratio for $k$-center clustering.

\begin{theorem}\label{theo-center-real}
For $k$-center clustering in Definition~\ref{def-center-real}, the expected number of iterations of the GSEMO using Eq.~(\ref{eq-bi-center}), until achieving a $2$-approximation ratio, is at most $ek^2n-ek(n-1)$.
\end{theorem}
\begin{proof}
The proof can be accomplished by following that of Theorem~\ref{theo-center}. The only difference is that we need to show that the solution $\bm{x}$ corresponding to $J_{\max}=k$ satisfies
\begin{align}\label{eq-center-real-1}
   \max\nolimits_{m \in [k]} \max\nolimits_{\bm{v} \in S_m} d(\bm{v},\bm{v}^*_m) \leq 2\cdot \mathrm{OPT},
\end{align}
instead of Eq.~(\ref{eq-center-14}). Note that $\bm{v}^*_m \in X$ serves as the center of the cluster $S_m=\{\bm{v} \in \mathcal{D} \mid \bm{v}^*_m =\arg\min_{\bm{u} \in X} d(\bm{v},\bm{u})\}$, where $X=\{\bm{v}^*_1,\bm{v}^*_2,\ldots,\bm{v}^*_{k}\}$ represents the $k$ data points corresponding to those 1-bits in $\bm{x}$; $\mathrm{OPT}$ denotes the optimal value of Eq.~(\ref{eq-center-real}). As $\mathcal{D}=\cup^k_{m=1} S_m$, we have
\begin{align}
    &\max_{\bm{v} \in \mathcal{D}} d(\bm{v},X)=\max_{m \in [k]} \max_{\bm{v} \in S_m}d(\bm{v},X)\nonumber\\
    &=\max_{m \in [k]} \max_{\bm{v} \in S_m}\min_{\bm{v}^*_j\in X} d(\bm{v},\bm{v}^*_j)=\max_{m \in [k]} \max_{\bm{v} \in S_m} d(\bm{v},\bm{v}^*_m),\label{eq-center-real-2}
\end{align}
where the second equality holds because $d(\bm{v},X)$ is the distance between $\bm{v}$ and its closest point in $X=\{\bm{v}^*_1,\bm{v}^*_2,\ldots,\bm{v}^*_{k}\}$, and the last equality holds because $\bm{v} \in S_m$ implies that the closest point of $\bm{v}$ in $X$ is $\bm{v}^*_m$. Eqs.~(\ref{eq-center-real-1}) and~(\ref{eq-center-real-2}) lead to
\begin{align}
\max\nolimits_{\bm{v} \in \mathcal{D}} d(\bm{v},X)\leq 2\cdot \mathrm{OPT},\nonumber
\end{align}
implying a $2$-approximation ratio for $k$-center clustering.

Next, we are to prove Eq.~(\ref{eq-center-real-1}). Let \begin{align}\label{eq-center-real-3}\max\nolimits_{m \in [k]} \max\nolimits_{\bm{v} \in S_m} d(\bm{v},\bm{v}^*_m)=h,\end{align} and we use $\bm{v}^*$ to denote the point which has the distance $h$ to its center. Eqs.~(\ref{eq-center-11}) and~(\ref{eq-center-12}) in the proof of Theorem~\ref{theo-center} have shown that the distance between any two points in $X \cup \{\bm{v}^*\}=\{\bm{v}^*_1,\bm{v}^*_2,\ldots,\bm{v}^*_{k},\bm{v}^*\}$ is at least $h$. Let $\bm{x}' \subseteq \mathcal{D}$ denote any solution with size $k$, and the corresponding $k$ centers and $k$ clusters are denoted as $\bm{v}'_1,\bm{v}'_2,\ldots,\bm{v}'_k$ and $S'_1,S'_2,\ldots,S'_k$, respectively. As $|X \cup \{\bm{v}^*\}|=k+1$, there must exist one cluster $S'_m$ containing at least two points (denoted as $\bm{u}_1$ and $\bm{u}_2$) in $X \cup \{\bm{v}^*\}$. If the center $\bm{v}'_m$ of $S'_m$ is one of these two points, we have $\max\nolimits_{\bm{v} \in S'_m} d(\bm{v},\bm{v}'_m) \geq h$. Otherwise, by the triangle inequality, we have $d(\bm{u}_1,\bm{v}'_m)+d(\bm{u}_2,\bm{v}'_m) \geq d(\bm{u}_1,\bm{u}_2)\geq h$, leading to $\max\nolimits_{\bm{v} \in S'_m} d(\bm{v},\bm{v}'_m) \geq h/2$ . Thus,
\begin{align}\max\nolimits_{m \in [k]} \max\nolimits_{\bm{v} \in S'_m} d(\bm{v},\bm{v}'_m)\geq h/2,\nonumber\end{align}
implying that the optimal objective value is at least $h/2$, i.e.,
\begin{align}\label{eq-center-real-4}
    h/2 \leq \mathrm{OPT}.
\end{align}
Combining Eqs.~(\refeq{eq-center-real-3}) and~(\refeq{eq-center-real-4}) leads to Eq.~(\ref{eq-center-real-1}). Thus, the theorem holds.
\end{proof}

\section{Theoretical Analysis of The GSEMO \\for discrete $k$-median Clustering}\label{sec-median}

The discrete $k$-median clustering problem in Definition~\ref{def-median} is to select a subset $X$ of $k$ data points from $F$, such that the sum of the distance of each data point in $\mathcal{D}$ to its nearest point in $X$ is minimized. To apply the GSEMO to solve discrete $k$-median clustering, we use a Boolean vector $\bm{x} \in \{0,1\}^{|F|}$ to represent a subset $X$ of $F$, where the $i$-th bit $x_i=1$ iff the $i$-th point in $F$ belongs to $X$. In the following analysis, we will not distinguish $\bm{x}\in \{0,1\}^{|F|}$ and its corresponding subset $X \subseteq F$ for convenience. Then, the original problem in Definition~\ref{def-median} is reformulated as a bi-objective maximization problem
\begin{align}\label{eq-bi-median}
&\max\nolimits_{\bm{x} \in \{0,1\}^{|F|}} \;\; (f_1(\bm{x}),f_2(\bm{x})),\\
&\text{where}\;\begin{cases}\nonumber
f_1(\bm{x}) = -\sum_{\bm{v}_i \in \mathcal{D}} d(\bm{v}_i,\bm{x}),\\
f_2(\bm x) = |\bm{x}|.
\end{cases}
\end{align}
Note that $d(\bm{v}_i,\bm{x})=\min \{d(\bm{v}_i,\bm{u}) \mid \bm{u} \in \bm{x}\}$ is the distance between $\bm{v}_i$ and its closest point in $\bm{x}$. That is, the GSEMO is to minimize the original objective function $\sum_{\bm{v}_i \in \mathcal{D}} d(\bm{v}_i,\bm{x})$ and maximize the subset size $|\bm{x}|$ simultaneously. To be well defined, we set the $f_1$ value to $-\infty$ for the all-0s vector $\bm{0}$. The Boolean-vector solutions with size larger than $k$ are excluded during the running of the GSEMO. When the GSEMO terminates, the Boolean-vector solution with size $k$ in the population will be output as the final solution.

For each $\bm{v}_i \in \mathcal{D}$, let $d^{\max}_i$ and $d^{\min}_i$ denote the distance between $\bm{v}_i$ and its farthest and closest (excluding itself) points in $F$, respectively. That is, $d^{\max}_i=\max \{d(\bm{v}_i,\bm{u}) \mid \bm{u} \in F\}$ and $d^{\min}_i=\min \{d(\bm{v}_i,\bm{u}) \mid \bm{u} \in F\setminus \{\bm{v}_i\}\}$. We use $d^{\min}_{(1)},\ldots,d^{\min}_{(n)}$ to denote a permutation of $d^{\min}_{1},\ldots,d^{\min}_{n}$ in ascending order. Theorem~\ref{theo-median} shows that the GSEMO achieves a $\frac{1}{1-\epsilon} \left(3+\frac{2}{p}\right)$-approximation ratio after running at most $O\left(\frac{k^2|F|^{2p}}{\epsilon}\log \frac{\sum^n_{i=1} d^{\max}_{i}}{\sum^{n-k}_{i=1} d^{\min}_{(i)}}\right)$ expected number of iterations. This implies that the Boolean-vector solution $\bm{x}$ output by the GSEMO satisfies $\sum_{\bm{v}_i \in \mathcal{D}} d(\bm{v}_i,\bm{x}) \leq \frac{1}{1-\epsilon} \left(3+\frac{2}{p}\right)\cdot \mathrm{OPT}$, where $\mathrm{OPT}$ denotes the optimal value of Eq.~(\ref{eq-median}). Note that the required expected number of iterations is polynomial in $|F|^p$, $1/\epsilon$, $\log n$ and $\log (\max_{i}d^{\max}_{i}/\min_{i}d^{\min}_{i})$.

The proof idea of Theorem~\ref{theo-median} is mainly to show that the GSEMO can simulate the process of local search in~\cite{arya2004local}. Here, we first prove general conditions under which an evolutionary process can simulate local search to achieve an approximation guarantee. Definitions~\ref{def-approx-local} and~\ref{def-local-approx} together present an $\alpha$-approximated problem by local search. That is, it is always possible to improve a feasible solution $\bm{x} \in \mathcal{X} \subseteq \{0,1\}^r$ by deleting at most $p$ points from $\bm{x}$ and inserting at most $q$ new points into $\bm{x}$, until achieving an $\alpha$-approximation ratio.

\begin{definition}[$(1-\delta)$-Approximate Local Optimum]\label{def-approx-local} Given a pseudo-Boolean problem $f: \mathcal{X} \rightarrow \mathbb{R}^+$ to be minimized, where $\mathcal{X} \subseteq \{0,1\}^r$ is the feasible solution space (i.e., the set of solutions satisfying the constraints), a feasible solution $\bm{x}$ is called a $(1-\delta)$-approximate local optimum if 
\begin{align*}
f(\bm{x}') > (1-\delta)\cdot f(\bm{x})
\end{align*}
for any feasible solution $\bm{x}'$ with $|\bm{x}\setminus \bm{x}'| \leq p$ and $|\bm{x}'\setminus \bm{x}| \leq q$, where $0 <\delta < 1$, and $p+q\geq 1$.
\end{definition}

\begin{definition}[$\alpha$-Approximation by Local Search]\label{def-local-approx} A pseudo-Boolean minimization problem $f: \mathcal{X} \rightarrow \mathbb{R}^+$ is $\alpha$-approximated by local search if for any $(1-\delta)$-approximate local optimum $\bm{x}$, it holds that
\begin{align*}
    f(\bm{x}) \leq \alpha\cdot \mathrm{OPT},
\end{align*}
where $\alpha\; (\alpha \geq 1)$ depends on $\delta$, and $\mathrm{OPT}$ denotes the optimal function value.
\end{definition}

We consider a typical evolutionary process as presented in Definition~\ref{def-evol-proc}, which covers a large class of EAs.

\begin{definition}[Evolutionary Process]\label{def-evol-proc}
An evolutionary process starts from a set of solutions (called a population), and iteratively improves the population by parent selection, reproduction, and survivor selection. Each iteration can be generally characterized by the following steps:
\begin{enumerate}
  \item \textbf{Parent Selection.} Some solutions $Q_i$ are selected from the population $P$ by using a parent selection strategy;
  \item \textbf{Reproduction.} Offspring solutions $R_i$ are generated by applying some reproduction operators to the selected parent solutions $Q_i$;
  \item Repeat the above process for $\lambda$ times;
  \item \textbf{Survivor Selection.} Select some solutions from the population $P$ and the newly generated offspring solutions $\cup^{\lambda}_{i=1} R_i$ to form the next population.
\end{enumerate}
\end{definition}

Lemma~\ref{lemma-general-condition} provides the conditions of an evolutionary process, required to achieve an $\alpha$-approximation ratio by simulating local search. It also gives the expected number of iterations of the evolutionary process. This lemma will be frequently used in the following analysis. Furthermore, it may be of independent interest for analyzing the approximation ability of EAs, e.g., it has been implicitly used in~\cite{friedrich2015maximizing,qian2019maximizing,qian2022result}.

\begin{lemma}\label{lemma-general-condition}
Given an $\alpha$-approximated problem by local search as presented in Definition~\ref{def-local-approx}, if an evolutionary process in Definition~\ref{def-evol-proc} satisfies the following conditions:
\begin{enumerate}
\item in parent selection, the best feasible solution is selected from the population with probability at least $\mathrm{Pr}_{\mathrm{sel}}$;
\item in reproduction, for any $i\leq p$ and $j\leq q$, a parent solution is flipped by $i$ specific 1-bits and $j$ specific 0-bits  with probability at least $\mathrm{Pr}_{\mathrm{rep}}(p,q)$;
\item in survivor selection, the best feasible solution generated so far is always kept,
\end{enumerate}
then starting from a population $P$, it achieves an approximation ratio of $\alpha$ after running at most \begin{align*}
\frac{1}{1-(1-\mathrm{Pr}_{\mathrm{sel}} \cdot \mathrm{Pr}_{\mathrm{rep}}(p,q))^{\lambda}} \cdot O\left(\frac{1}{\delta} \log  \frac{f_{\mathrm{init}}}{\mathrm{OPT}}\right)
\end{align*} expected number of iterations, where $f_{\mathrm{init}} =\min\{f(\bm{x})\mid \bm{x} \;\text{is a feasible solution in} \; P\}$. 
\end{lemma}
\begin{proof}
Because the problem is $\alpha$-approximated by local search, we only need to analyze the expected number of iterations until generating a $(1-\delta)$-approximate local optimum $\bm{x}$, which satisfies $f(\bm{x}) \leq \alpha \cdot \mathrm{OPT}$, achieving the desired approximation ratio. 

Let $\hat{\bm{x}}$ denote the best feasible solution in the population $P$. We consider $f(\hat{\bm{x}})$, which obviously will not increase, because the evolutionary process keeps the best feasible solution generated so far (i.e., the 3rd condition). As long as $\hat{\bm{x}}$ is not a $(1-\delta)$-approximate local optimum, we know from Definition~\ref{def-approx-local} that a new feasible offspring solution $\bm{x}'$ with 
\begin{align*}
f(\bm{x}')\leq (1-\delta)\cdot f(\hat{\bm{x}})
\end{align*} 
can be generated through selecting $\hat{\bm{x}}$ in parent selection and flipping at most $p$ specific 1-bits and $q$ specific 0-bits (i.e., deleting at most $p$ points inside $\hat{\bm{x}}$ and inserting at most $q$ new points into $\hat{\bm{x}}$) in reproduction, the probability of which is at least $\mathrm{Pr}_{\mathrm{sel}} \cdot \mathrm{Pr}_{\mathrm{rep}}(p,q)$ by the 1st and 2nd conditions. As parent selection and reproduction are repeated for $\lambda$ times in each iteration of the evolutionary process, as shown in Definition~\ref{def-evol-proc}, the probability of generating $\bm{x}'$ in each iteration is at least
\begin{align*}
1-(1-\mathrm{Pr}_{\mathrm{sel}} \cdot \mathrm{Pr}_{\mathrm{rep}}(p,q))^{\lambda}.
\end{align*} 
We know from the 3rd condition that the evolutionary process keeps the best feasible solution generated so far, implying that the best feasible solution in the next population is at least as good as $\bm{x}'$, and thus $f(\hat{\bm{x}})$ decreases by at least a factor of $1/(1-\delta)$. Such a decrease on $f(\hat{\bm{x}})$ is called a successful step. Thus, a successful step needs at most 
\begin{align*}
1/\left(1-(1-\mathrm{Pr}_{\mathrm{sel}} \cdot \mathrm{Pr}_{\mathrm{rep}}(p,q))^{\lambda}\right)
\end{align*}  expected number of iterations. Until generating a $(1-\delta)$-approximate local optimum, the required number of successful steps is at most 
\begin{align*}
\log_{\frac{1}{1-\delta}} \frac{f_{\mathrm{init}}}{\mathrm{OPT}}=O\left(\frac{1}{\delta} \log  \frac{f_{\mathrm{init}}}{\mathrm{OPT}}\right),
\end{align*} 
where $f_{\mathrm{init}}$ is the $f$ value of the best feasible solution in the initial population. Thus, the expected number of iterations until generating a $(1-\delta)$-approximate local optimum is at most
\begin{equation*}
\begin{aligned}
\frac{1}{1-(1-\mathrm{Pr}_{\mathrm{sel}} \cdot \mathrm{Pr}_{\mathrm{rep}}(p,q))^{\lambda}} \cdot O\left(\frac{1}{\delta} \log  \frac{f_{\mathrm{init}}}{\mathrm{OPT}}\right).
\end{aligned}
\end{equation*}
\end{proof}

Lemma~\ref{lemma-median} shows that the discrete $k$-median clustering problem is $\frac{1}{1-\epsilon} \left(3+\frac{2}{p}\right)$-approximated by a $(1-\frac{\epsilon}{k})$-approximate local optimum. That is, $\delta=\epsilon/k$ and $p=q$ in Definition~\ref{def-approx-local}, and $\alpha=\frac{1}{1-\epsilon} \left(3+\frac{2}{p}\right)$ in Definition~\ref{def-local-approx}. This lemma is inspired from the analysis of local search in~\cite{arya2004local}, which tries to repeatedly improve a subset of $k$ points by dropping at most $p$ points and adding the same number of new points.

\begin{lemma}\label{lemma-median}
Let $X$ be a subset of $F$ with size $k$. If no subset $X'$ of $F$ with the objective value 
\begin{align*}
\sum_{\bm{v}_i \in \mathcal{D}} d(\bm{v}_i,X') \leq \left(1-\frac{\epsilon}{k}\right)\cdot \sum_{\bm{v}_i \in \mathcal{D}} d(\bm{v}_i,X)
\end{align*}
can be achieved by deleting at most $p$ points inside $X$ and inserting the same number of points outside into $X$, then \begin{align}\label{eq-median-3}
\sum_{\bm{v}_i \in \mathcal{D}} d(\bm{v}_i,X) \leq \frac{1}{1-\epsilon} \left(3+\frac{2}{p}\right)\cdot \mathrm{OPT},\end{align}
where $\epsilon >0$ and $p \geq 1$.
\end{lemma}
\begin{proof}
In Section~3 of~\cite{arya2004local}, the special case with $\epsilon = 0$ has been proved. Their proof uses a series of multiswaps of $X$, where a multiswap of $X$ deletes at most $p$ points inside $X$ and inserts the same number of points outside into $X$. Let $Q$ denote the collection of all the sets generated by these multiswaps of $X$. Considering the condition of this lemma with $\epsilon = 0$, it holds that for each $X' \in Q$,
\begin{align}\label{eq-median-1}
    \sum\nolimits_{\bm{v}_i \in \mathcal{D}} d(\bm{v}_i,X') - \sum\nolimits_{\bm{v}_i \in \mathcal{D}} d(\bm{v}_i,X) >0.
\end{align}
By assigning a positive real weight $w(X')$ with each $X' \in Q$, the proof in~\cite{arya2004local} derives  
\begin{align}\label{eq-median-2}
    &\sum_{X' \in Q} w(X')\cdot \left(\sum_{\bm{v}_i \in \mathcal{D}} d(\bm{v}_i,X') - \sum_{\bm{v}_i \in \mathcal{D}} d(\bm{v}_i,X)\right) \\
    &\leq (3+2/p)\cdot \mathrm{OPT} - \sum\nolimits_{\bm{v}_i \in \mathcal{D}} d(\bm{v}_i,X),\nonumber
\end{align}
where $\sum_{X' \in Q} w(X')\leq k$. Combining Eqs.~(\ref{eq-median-1}) and~(\ref{eq-median-2}) leads to \begin{align*}
\sum\nolimits_{\bm{v}_i \in \mathcal{D}} d(\bm{v}_i,X) \leq (3+2/p)\cdot \mathrm{OPT},
\end{align*} 
i.e., Eq.~(\ref{eq-median-3}) with $\epsilon=0$.

We adapt their proof to the case of $\epsilon >0$. According to the condition of this lemma, we now have, for each $X' \in Q$,
\begin{align}\label{eq-median-4}
    \sum_{\bm{v}_i \in \mathcal{D}} d(\bm{v}_i,X') - \sum_{\bm{v}_i \in \mathcal{D}} d(\bm{v}_i,X) > -\frac{\epsilon}{k} \sum_{\bm{v}_i \in \mathcal{D}} d(\bm{v}_i,X).
\end{align}
Combining Eqs.~(\ref{eq-median-2}) and~(\ref{eq-median-4}) leads to
\begin{align*}
    &(3+2/p)\cdot \mathrm{OPT} - \sum\nolimits_{\bm{v}_i \in \mathcal{D}} d(\bm{v}_i,X) \\
    &\geq \sum_{X' \in Q} w(X')\cdot \left(-\frac{\epsilon}{k} \sum_{\bm{v}_i \in \mathcal{D}} d(\bm{v}_i,X)\right)\geq -\epsilon \sum_{\bm{v}_i \in \mathcal{D}} d(\bm{v}_i,X),
\end{align*}
where the last inequality holds by $\sum_{X' \in Q} w(X')\leq k$. Thus, Eq.~(\ref{eq-median-3}) holds.
\end{proof}

By applying Lemma~\ref{lemma-general-condition} to the GSEMO (i.e., proving that the GSEMO satisfies the three conditions in Lemma~\ref{lemma-general-condition}), we have Theorem~\ref{theo-median}. But in the proof, we first need to analyze the expected number of iterations until generating a solution with size $k$, i.e., a feasible solution, as the GSEMO starts from the initial solution with all 0s.

\begin{theorem}\label{theo-median}
For discrete $k$-median clustering in Definition~\ref{def-median}, the expected number of iterations of the GSEMO using Eq.~(\ref{eq-bi-median}), until achieving a $\frac{1}{1-\epsilon} \left(3+\frac{2}{p}\right)$-approximation ratio, is at most $O\left(\frac{k^2|F|^{2p}}{\epsilon}\log \frac{\sum^n_{i=1} d^{\max}_{i}}{\sum^{n-k}_{i=1} d^{\min}_{(i)}}\right)$, where $\epsilon>0$, $p\geq 1$, $d^{\max}_i=\max \{d(\bm{v}_i,\bm{u}) \mid \bm{u} \in F\}$, $d^{\min}_{(1)},\ldots,d^{\min}_{(n)}$ is a permutation of $d^{\min}_{1},\ldots,d^{\min}_{n}$ in ascending order, and $d^{\min}_i=\min \{d(\bm{v}_i,\bm{u}) \mid \bm{u} \in F\setminus \{\bm{v}_i\}\}$.
\end{theorem}
\begin{proof}
The optimization process is divided into two phases: (1)~starts from the initial solution $\bm{0}$ and finishes after finding a solution with size $k$; (2)~starts after phase~(1) and finishes after achieving the desired approximation ratio $\frac{1}{1-\epsilon} \left(3+\frac{2}{p}\right)$. We analyze the upper bound on the expected number of iterations required by each phase, respectively, and then sum them up to get an upper bound on the total expected number of iterations of the GSEMO.

For phase (1), we consider the maximum number of 1-bits of the solutions in the population $P$, denoted by $J_{\max}$. That is, $J_{\max}=\max\{|\bm{x}| \mid \bm{x} \in P\}$. Obviously, $J_{\max}=k$ implies that a solution with size $k$ has been found, i.e., the goal of phase~(1) has been reached. As the GSEMO starts from the all-0s vector $\bm{0}$, $J_{\max}$ is initially 0. Assume that currently $J_{\max}=i < k$, and let $\bm{x}$ be the corresponding solution, i.e., $|\bm{x}|=i$. $J_{\max}$ will not decrease because $\bm{x}$ cannot be weakly dominated by a solution with less 1-bits. By selecting $\bm{x}$ in line~3 of Algorithm~\ref{algo:GSEMO} and flipping only one 0-bit of $\bm{x}$ (i.e., adding a new point into $\bm{x}$) in line~4, which occur with probability $(1/|P|) \cdot (|F|-i)\cdot (1/|F|)(1-1/|F|)^{|F|-1} \geq (|F|-i)/(e|P||F|)$, a new solution $\bm{x}'$ with $|\bm{x}'|=i+1$ can be generated in one iteration of the GSEMO. The population size $|P|$ is at most $i+1$, because the second objective $f_2$ can only take values $0,1,\ldots,i$, and the solutions in $P$ are incomparable. Thus, the probability of generating a new solution $\bm{x}'$ with $|\bm{x}'|=i+1$ in one iteration is at least $(|F|-i)/(e(i+1)|F|)$. Because the newly generated solution $\bm{x}'$ now has the largest number of 1-bits and no solution in $P$ can dominate it, it will be included into $P$, making $J_{\max}=i+1$. This implies that the probability of increasing $J_{\max}$ in one iteration of the GSEMO is at least $(|F|-i)/(e(i+1)|F|)$, i.e., 
$$\mathrm{Pr}[J_{\max} \;\text{increases in one iteration} \mid J_{\max}\!=\!i]\geq \frac{|F|-i}{e(i\!+\!1)|F|}.$$We then get that the expected number of iterations of phase~(1) (i.e., to make $J_{\max}$ reach $k$) is at most
\begin{align}\label{eq-median-6}
\sum^{k-1}_{i=0} \frac{e(i+1)|F|}{|F|-i} \leq ek|F| \ln \frac{|F|}{|F|-k},
\end{align}
where we assume $k<|F|$, which obviously holds in practice. Note that the population $P$ will always contain a solution with size $k$ once generated, since it has the largest $f_2$ value $k$ and can be weakly dominated by only other solutions with size $k$.

Next, we consider phase~(2) by applying Lemma~\ref{lemma-general-condition}. According to Lemma~\ref{lemma-median}, we know that the discrete $k$-median clustering problem is $\frac{1}{1-\epsilon} \left(3+\frac{2}{p}\right)$-approximated by a $(1-\frac{\epsilon}{k})$-approximate local optimum. That is, $\delta=\epsilon/k$ in Definition~\ref{def-approx-local}, and $\alpha=\frac{1}{1-\epsilon} \left(3+\frac{2}{p}\right)$ in Definition~\ref{def-local-approx}. Furthermore, the parameters $p$ and $q$ in Definition~\ref{def-approx-local} are equal here. As presented in Algorithm~\ref{algo:GSEMO}, the GSEMO performs the following three steps in each iteration: selects a solution from the current population uniformly at random, applies the bit-wise mutation operator only, and uses the generated offspring solution to update the population. It is obvious that the GSEMO follows the evolutionary process with $\lambda=1$ in Definition~\ref{def-evol-proc}.

Now we are to analyze the three conditions of Lemma~\ref{lemma-general-condition}. Here, a feasible solution is a solution with size $k$, i.e., $k$ clusters. During phase~(2), the population $P$ always contains only one feasible solution, denoted as $\hat{\bm{x}}$. Because $\hat{\bm{x}}$ can be dominated by only other solutions with size $k$ and larger $f_1$ values, where $f_1(\hat{\bm{x}})=-\sum_{\bm{v}_i \in \mathcal{D}} d(\bm{v}_i,\hat{\bm{x}})$ as in Eq.~(\ref{eq-bi-median}), the 3rd condition of Lemma~\ref{lemma-general-condition} is satisfied, i.e., the best feasible solution generated so far is always kept in the population. By uniform parent selection in line~3 of Algorithm~\ref{algo:GSEMO}, the probability of selecting $\hat{\bm{x}}$ for reproduction is $1/|P|$. As analyzed in the proof of Theorem~\ref{theo-center}, the population size $|P|$ is at most $k+1$, because the second objective $f_2$ can only take values $0,1,\ldots,k$, and the solutions in $P$ are incomparable. Note that the solutions with size larger than $k$ are excluded during the running of the GSEMO. In fact, $|P| \leq k$, because the all-0s solution (having $f_1(\bm{0})=-\infty$ and $f_2(\bm{0})=0$) is dominated by any other solution, and will not exist in $P$ once a solution with size larger than 0 has been generated. Thus, the 1st condition of Lemma~\ref{lemma-general-condition} is satisfied with $\mathrm{Pr}_{\mathrm{sel}}=1/k$. By bit-wise mutation in line~4 of Algorithm~\ref{algo:GSEMO}, the probability of flipping at most $p$ specific 1-bits and $p$ specific 0-bits is at least $(1/|F|^{2p})(1-1/|F|)^{|F|-2p}\geq 1/(e|F|^{2p})$, where the inequality holds by $p\geq 1$. This implies that the 2nd condition of Lemma~\ref{lemma-general-condition} is satisfied with $\mathrm{Pr}_{\mathrm{rep}}(p,q)=1/(e|F|^{2p})$. Thus, by Lemma~\ref{lemma-general-condition}, the expected number of iterations of phase~(2) is at most
\begin{align}\label{eq-median-5}
&\frac{1}{1-(1-\mathrm{Pr}_{\mathrm{sel}} \cdot \mathrm{Pr}_{\mathrm{rep}}(p,q))^{\lambda}} \cdot O\left(\frac{1}{\delta} \log  \frac{f_{\mathrm{init}}}{\mathrm{OPT}}\right)\nonumber\\
&= ek|F|^{2p} \cdot O\left(\frac{k}{\epsilon} \log  \frac{f_{\mathrm{init}}}{\mathrm{OPT}}\right),
\end{align} 
where the equality holds by $\mathrm{Pr}_{\mathrm{sel}}=1/k$, $\mathrm{Pr}_{\mathrm{rep}}(p,q)=1/(e|F|^{2p})$, $\lambda=1$, and $\delta=\epsilon/k$.

For any solution $\bm{x}$ with size $k$, as $d^{\max}_i=\max \{d(\bm{v}_i,\bm{u}) \mid \bm{u} \in F\}$ and $\bm{x} \subseteq F$, we have $d(\bm{v}_i,\bm{x})=\min \{d(\bm{v}_i,\bm{u}) \mid \bm{u} \in \bm{x}\} \leq d^{\max}_{i}$. Thus, $\sum_{\bm{v}_i \in \mathcal{D}} d(\bm{v}_i,\bm{x})\leq \sum^n_{i=1} d^{\max}_{i}$, implying that the objective function value of the first generated feasible solution (i.e., the first generated solution with size $k$)
\begin{align}\label{eq-median-21}f_{\mathrm{init}} \leq \sum\nolimits^n_{i=1} d^{\max}_{i}.
\end{align} Furthermore,
\begin{align*}
\sum_{\bm{v}_i \in \mathcal{D}} d(\bm{v}_i,\bm{x})=
\sum_{\bm{v}_i \in \mathcal{D}\setminus \bm{x}} d(\bm{v}_i,\bm{x}) \geq 
\sum_{\bm{v}_i \in \mathcal{D}\setminus \bm{x}} d^{\min}_i \geq \sum^{n-k}_{i=1} d^{\min}_{(i)},
\end{align*}
where the first inequality holds by $d(\bm{v}_i,\bm{x})=\min \{d(\bm{v}_i,\bm{u}) \mid \bm{u} \in \bm{x}\} \geq \min \{d(\bm{v}_i,\bm{u}) \mid \bm{u} \in F \setminus \{\bm{v}_i\}\}=d^{\min}_{i}$ due to $\bm{x} \subseteq F \setminus \{\bm{v}_i\}$, and the last inequality holds because $|\mathcal{D}\setminus \bm{x}|\geq n-k$, and $d^{\min}_{(1)},d^{\min}_{(2)},\ldots,d^{\min}_{(n)}$ is a permutation of $d^{\min}_{1},d^{\min}_{2},\ldots,d^{\min}_{n}$ in ascending order. This implies
\begin{align}\label{eq-median-20}
\mathrm{OPT} \geq \sum\nolimits^{n-k}_{i=1} d^{\min}_{(i)}.
\end{align}
Applying Eqs.~(\ref{eq-median-21}) and~(\ref{eq-median-20}) to Eq.~(\ref{eq-median-5}), the expected number of iterations of phase~(2) is at most
\begin{align}\label{eq-median-7}
ek|F|^{2p} \cdot O\left(\frac{k}{\epsilon}\log \frac{\sum^n_{i=1} d^{\max}_{i}}{\sum^{n-k}_{i=1} d^{\min}_{(i)}}\right).
\end{align}

By combining the expected number of iterations (i.e., Eqs.~(\ref{eq-median-6}) and~(\ref{eq-median-7})) in the above two phases, we can conclude that the GSEMO requires at most
\begin{align*}
&ek|F| \ln \frac{|F|}{|F|-k}+ek|F|^{2p} \cdot O\left(\frac{k}{\epsilon}\log \frac{\sum^n_{i=1} d^{\max}_{i}}{{\sum^{n-k}_{i=1} d^{\min}_{(i)}}}\right)\\
&=O\left(\frac{k^2|F|^{2p}}{\epsilon}\log \frac{\sum^n_{i=1} d^{\max}_{i}}{\sum^{n-k}_{i=1} d^{\min}_{(i)}}\right)
\end{align*}
iterations in expectation to achieve an approximation ratio of $\frac{1}{1-\epsilon} \left(3+\frac{2}{p}\right)$. Thus, the theorem holds.
\end{proof}

We have considered discrete $k$-median clustering, where the centers are selected from a specific set $F$. When the centers can be placed anywhere in space, we cannot directly use the way of Boolean-vector solution representation, because the vector length will be infinite. But we can select the $k$ centers from a discrete set $C$ of candidate centers of size $O(k^2\epsilon^{-2l}\log^2 n)$, which contains a solution $X \subseteq C$ with size $k$ such that $\sum_{\bm{v}_i \in \mathcal{D}} d(\bm{v}_i,X)\leq (1+\epsilon)\cdot \mathrm{OPT}$, as shown in Lemma~5.3 of~\cite{har2004coresets}. This idea of using an $\epsilon$-approximate centroid set will also be used in the following analysis for $k$-means clustering.

\section{Theoretical Analysis of The GSEMO\\ for $k$-Means Clustering}\label{sec-means}

The $k$-means clustering problem in Definition~\ref{def-means} is to determine a set $X$ of $k$ points (also called centers) in $\mathbb{R}^{l}$, to minimize the sum of the squared Euclidean distance from each point in $\mathcal{D}$ to its closest center in $X$, i.e., $\sum_{\bm{v}_i \in \mathcal{D}} \min_{\bm{u} \in X} \|\bm{v}_i-\bm{u}\|^2$. Though the centers can be placed anywhere, it has been proved in~\cite{matouvsek2000approximate} that there is a set $C$ of $O(n\epsilon^{-l}\log(1/\epsilon))$ candidate centers, which contains an approximately optimal solution, as shown in Lemma~\ref{lemma-means-1}.

\begin{lemma}[Theorem~4.4 in~\cite{matouvsek2000approximate}]\label{lemma-means-1} Given a set of $n$ data points $\mathcal{D}=\{\bm{v}_1,\bm{v}_2,\ldots,\bm{v}_n\}$ in $\mathbb{R}^{l}$, a set $C$ of size $O(n\epsilon^{-l}\log(1/\epsilon))$ can be constructed in time $O(n\log n+n\epsilon^{-l}\log(1/\epsilon))$, satisfying that there is a subset $X$ of $C$ with size $k$ such that 
\begin{align*}
     \sum_{\bm{v}_i \in \mathcal{D}} \min_{\bm{u} \in X} \|\bm{v}_i-\bm{u}\|^2 \leq (1+\epsilon)\cdot \mathrm{OPT},
\end{align*}
where $\epsilon>0$, and $\mathrm{OPT}$ denotes the optimal value of Eq.~(\ref{eq-means}).
\end{lemma}

To apply the GSEMO to solve $k$-means clustering, we select $k$ centers directly from the candidate set $C$ of size $O(n\epsilon^{-l}\log(1/\epsilon))$, as Lemma~\ref{lemma-means-1} has shown that $C$ contains an approximately optimal solution. A subset $X$ of $C$ is represented by a Boolean vector $\bm{x} \in \{0,1\}^{|C|}$, where the $i$-th bit $x_i=1$ iff the $i$-th center in $C$ belongs to $X$. Similar to the bi-objective reformulation Eq.~(\ref{eq-bi-median}) for discrete $k$-median clustering, the original $k$-means problem in Definition~\ref{def-means} is reformulated as a bi-objective maximization problem
\begin{align}\label{eq-bi-means}
&\max\nolimits_{\bm{x} \in \{0,1\}^{|C|}} \;\; (f_1(\bm{x}),f_2(\bm{x})),\\
&\text{where}\;\begin{cases}\nonumber
f_1(\bm{x}) = -\sum_{\bm{v}_i \in \mathcal{D}} \min_{\bm{u} \in X} \|\bm{v}_i-\bm{u}\|^2,\\
f_2(\bm x) = |\bm{x}|.
\end{cases}
\end{align}
We also set $f_1(\bm{0})$ to $-\infty$, and exclude the Boolean-vector solutions with size larger than $k$ during optimization. When terminated, the GSEMO outputs the Boolean-vector solution with size $k$ from the population.

Inspired from the analysis of local search in~\cite{kanungoa2004local}, we derive Lemma~\ref{lemma-means-2} which shows that it is always possible to improve a subset $X$ of $C$ with size $k$ by swapping at most $p$ centers in and out, until an approximation ratio of $\frac{1}{(1-\epsilon)^2} \left(3+\frac{2}{p}\right)^2$ has been achieved. Note that this approximation is with respect to the objective function value (denoted as $\mathrm{OPT}_C$) of the best subset of $C$ with size $k$. According to Lemma~\ref{lemma-means-1}, we have
\begin{align}\label{eq-means-7}
\mathrm{OPT}_C \leq (1+\epsilon)\cdot \mathrm{OPT},
\end{align}
where $\mathrm{OPT}$ is the optimal function value of the original problem in Definition~\ref{def-means}.

\begin{lemma}\label{lemma-means-2}
Let $X$ be a subset of $C$ with size $k$. If no subset $X'$ of $C$ with the objective value 
\begin{align*}
&\sum_{\bm{v}_i \in \mathcal{D}} \min_{\bm{u} \in X'} \|\bm{v}_i-\bm{u}\|^2 \\
&\leq \left(1-\left(1+\frac{1-\epsilon}{3+2/p}\right)\frac{\epsilon}{k}\right)\cdot \sum_{\bm{v}_i \in \mathcal{D}} \min_{\bm{u} \in X} \|\bm{v}_i-\bm{u}\|^2
\end{align*}
can be achieved by deleting at most $p$ points inside $X$ and inserting the same number of points outside into $X$, then \begin{align}\label{eq-means-3}
\sum_{\bm{v}_i \in \mathcal{D}} \min_{\bm{u} \in X} \|\bm{v}_i-\bm{u}\|^2 \leq \frac{1}{(1-\epsilon)^2} \left(3+\frac{2}{p}\right)^2\cdot \mathrm{OPT}_C.\end{align}
where $\epsilon >0$, $p \geq 1$, and $\mathrm{OPT}_C$ denotes the objective function value of the best subset of $C$ with size $k$.
\end{lemma}
\begin{proof}
Theorem~2.2 of~\cite{kanungoa2004local} shows the special case with $\epsilon = 0$. Similar to the analysis for discrete $k$-median clustering in~\cite{arya2004local} (which we have shown in the proof of Lemma~\ref{lemma-median}), their proof also relies on a collection $Q$ of the sets generated by performing different multiswaps of $X$, where each multiswap of $X$ swaps at most $p$ points in and out. By assigning a positive real weight $w(X')$ with each $X' \in Q$, satisfying $\sum_{X' \in Q} w(X')\leq k$, their proof derives  
\begin{align}\label{eq-means-1}
    &\sum_{X' \in Q} w(X')\left(\sum_{\bm{v}_i \in \mathcal{D}} \min_{\bm{u} \in X'} \|\bm{v}_i-\bm{u}\|^2 - \sum_{\bm{v}_i \in \mathcal{D}} \min_{\bm{u} \in X} \|\bm{v}_i-\bm{u}\|^2\right) \nonumber\\
    &\leq \left(\!3\!+\!\frac{2}{p}\right)\cdot \mathrm{OPT}_C - \left(\!1\!-\!\frac{2}{\alpha}\left(1\!+\!\frac{1}{p}\right)\!\!\right)\!\sum_{\bm{v}_i \in \mathcal{D}} \min_{\bm{u} \in X} \|\bm{v}_i\!-\!\bm{u}\|^2,
\end{align}
where\vspace{-1.5em} \begin{align}\label{eq-means-6}
\alpha^2=\sum_{\bm{v}_i \in \mathcal{D}} \min_{\bm{u} \in X} \|\bm{v}_i-\bm{u}\|^2/\mathrm{OPT}_C
\end{align}
just denotes the approximation ratio of the subset $X$ to $\mathrm{OPT}_C$. When $\epsilon = 0$, the condition of this lemma implies that
\begin{align}\label{eq-means-2}
    \forall X' \in Q, \sum_{\bm{v}_i \in \mathcal{D}} \min_{\bm{u} \in X'} \|\bm{v}_i-\bm{u}\|^2 - \sum_{\bm{v}_i \in \mathcal{D}} \min_{\bm{u} \in X} \|\bm{v}_i-\bm{u}\|^2 >0.
\end{align}
Combining Eqs.~(\ref{eq-means-1}) and~(\ref{eq-means-2}) leads to \begin{align*}
\left(1-\frac{2}{\alpha}\left(1+\frac{1}{p}\right)\!\!\right)\sum_{\bm{v}_i \in \mathcal{D}} \min_{\bm{u} \in X} \|\bm{v}_i\!-\!\bm{u}\|^2 \leq \left(3+\frac{2}{p}\right)\cdot \mathrm{OPT}_C,
\end{align*}
which is equivalent to $\alpha^2-2(1+1/p)\cdot \alpha-(3+2/p)\leq 0$. Thus, $\alpha \leq 3+2/p$, implying Eq.~(\ref{eq-means-3}) with $\epsilon=0$.

We adapt their proof to the case of $\epsilon >0$. According to the condition of this lemma, Eq.~(\ref{eq-means-2}) now changes to
\begin{align}\label{eq-means-4}
        \forall X' \in Q, &\sum_{\bm{v}_i \in \mathcal{D}} \min_{\bm{u} \in X'} \|\bm{v}_i-\bm{u}\|^2 - \sum_{\bm{v}_i \in \mathcal{D}} \min_{\bm{u} \in X} \|\bm{v}_i-\bm{u}\|^2 \\
        &> -\left(1+\frac{1-\epsilon}{3+2/p}\right)\frac{\epsilon}{k} \sum_{\bm{v}_i \in \mathcal{D}} \min_{\bm{u} \in X} \|\bm{v}_i-\bm{u}\|^2.\nonumber
\end{align}
Combining Eqs.~(\ref{eq-means-1}) and~(\ref{eq-means-4}) leads to
\begin{align}\label{eq-means-5}
    &\left(3+\frac{2}{p}\right)\cdot \mathrm{OPT}_C - \left(1-\frac{2}{\alpha}\left(1+\frac{1}{p}\right)\!\!\right)\sum_{\bm{v}_i \in \mathcal{D}} \min_{\bm{u} \in X} \|\bm{v}_i-\bm{u}\|^2 \nonumber\\
    &\geq \sum_{X' \in Q} w(X')\left(-\left(1+\frac{1-\epsilon}{3+2/p}\right)\frac{\epsilon}{k} \sum_{\bm{v}_i \in \mathcal{D}} \min_{\bm{u} \in X} \|\bm{v}_i-\bm{u}\|^2\right)\nonumber\\
    &\geq -\epsilon \left(1+\frac{1-\epsilon}{3+2/p}\right) \sum_{\bm{v}_i \in \mathcal{D}} \min_{\bm{u} \in X} \|\bm{v}_i-\bm{u}\|^2,
\end{align}
where the last inequality holds by $\sum_{X' \in Q} w(X')\leq k$. According to Eq.~(\ref{eq-means-6}), we substitute $\sum_{\bm{v}_i \in \mathcal{D}} \min_{\bm{u} \in X} \|\bm{v}_i-\bm{u}\|^2$ with $\alpha^2\cdot \mathrm{OPT}_C$ in Eq.~(\ref{eq-means-5}), leading to
\begin{align*}
\left(1-\frac{2}{\alpha}\left(1+\frac{1}{p}\right)-\epsilon \left(1+\frac{1-\epsilon}{3+2/p}\right)\right)\alpha^2\leq \left(3+\frac{2}{p}\right),
\end{align*}
which is equivalent to
\begin{align*}
(1-\epsilon) \left(1-\frac{\epsilon}{3+2/p}\right)\alpha^2-2\left(1+\frac{1}{p}\right)\alpha- \left(3+\frac{2}{p}\right)\leq 0.
\end{align*}
Thus, we have
\begin{align*}
\left((1-\epsilon)\alpha-\left(3+\frac{2}{p}\right)\right) \left(\left(1-\frac{\epsilon}{3+2/p}\right)\alpha+1\right) \leq 0,
\end{align*}
implying that $\alpha \leq (3+2/p)/(1-\epsilon)$. Thus, Eq.~(\ref{eq-means-3}) holds.
\end{proof}

Combining Lemma~\ref{lemma-means-2} and Eq.~(\ref{eq-means-7}) implies that the $k$-means clustering problem is $\frac{1+\epsilon}{(1-\epsilon)^2} \left(3+\frac{2}{p}\right)^2$-approximated by a $\left(1-\left(1+\frac{1-\epsilon}{3+2/p}\right)\frac{\epsilon}{k}\right)$-approximate local optimum. That is, $\delta=\left(1+\frac{1-\epsilon}{3+2/p}\right)\frac{\epsilon}{k}$ in Definition~\ref{def-approx-local}, and $\alpha=\frac{1+\epsilon}{(1-\epsilon)^2} \left(3+\frac{2}{p}\right)^2$ in Definition~\ref{def-local-approx}. Furthermore, the parameters $p$ and $q$ in Definition~\ref{def-approx-local} are equal here. By following the proof procedure of Theorem~\ref{theo-median}, we can apply Lemma~\ref{lemma-general-condition} to prove Theorem~\ref{theo-means}, showing a $\frac{1+\epsilon}{(1-\epsilon)^2} \left(3+\frac{2}{p}\right)^2$-approximation ratio of the GSEMO for $k$-means clustering, i.e., the solution $\bm{x}$ output by the GSEMO satisfies $$\sum_{\bm{v}_i \in \mathcal{D}} \min_{\bm{u} \in \bm{x}} \|\bm{v}_i-\bm{u}\|^2 \leq \frac{1+\epsilon}{(1-\epsilon)^2} \left(3+\frac{2}{p}\right)^2\cdot \mathrm{OPT}.$$The required expected number of iterations is polynomial in $n^p$, $1/\epsilon^{lp}$ and $\log (\max_{i}d^{\max}_{i}/\min_{i}d^{\min}_{i})$.

\begin{theorem}\label{theo-means}
For $k$-means clustering in Definition~\ref{def-means}, the expected number of iterations of the GSEMO using Eq.~(\ref{eq-bi-means}), until achieving a $\frac{1+\epsilon}{(1-\epsilon)^2} \left(3+\frac{2}{p}\right)^2$-approximation ratio, is at most $O\left(\frac{k^2|C|^{2p}}{\epsilon}\log \frac{\sum^n_{i=1} d^{\max}_{i}}{\sum^{n-k}_{i=1} d^{\min}_{(i)}}\right)$, where $\epsilon>0$, $p\geq 1$, the set $C$ is constructed as in Lemma~\ref{lemma-means-1} with $|C|=O(n\epsilon^{-l}\log(1/\epsilon))$, $d^{\max}_i=\max_{\bm{u} \in C} \|\bm{v}_i-\bm{u}\|^2$, $d^{\min}_{(1)},\ldots,d^{\min}_{(n)}$ is a permutation of $d^{\min}_{1},\ldots,d^{\min}_{n}$ in ascending order, and $d^{\min}_i=\min_{\bm{u} \in C\setminus \{\bm{v}_i\}} \|\bm{v}_i-\bm{u}\|^2$.
\end{theorem}
\begin{proof}
The proof is similar to that of Theorem~\ref{theo-median}. In phase~(1), the GSEMO requires at most $ek|C| \ln \frac{|C|}{|C|-k}$ expected number of iterations to find a solution with size $k$. Note that we have replaced the notation $F$ in the proof of Theorem~\ref{theo-median} with $C$ accordingly. The $k$-means problem is $\frac{1+\epsilon}{(1-\epsilon)^2} \left(3+\frac{2}{p}\right)^2$-approximated by a $\left(1-\left(1+\frac{1-\epsilon}{3+2/p}\right)\frac{\epsilon}{k}\right)$-approximate local optimum. Similar to the analysis of Eq.~(\ref{eq-median-7}) in the proof of Theorem~\ref{theo-median}, we can use Lemma~\ref{lemma-general-condition} to get that the expected number of iterations of phase~(2) is at most
\begin{align*}
&\frac{1}{1-(1-\mathrm{Pr}_{\mathrm{sel}} \cdot \mathrm{Pr}_{\mathrm{rep}}(p,q))^{\lambda}} \cdot   O\left(\frac{1}{\delta} \log  \frac{f_{\mathrm{init}}}{\mathrm{OPT}}\right)\nonumber\\
&=ek|C|^{2p}\cdot  O\left(\frac{k}{\epsilon}\log \frac{\sum^n_{i=1} d^{\max}_{i}}{\sum^{n-k}_{i=1} d^{\min}_{(i)}}\right),
\end{align*} 
where the equality holds by $\mathrm{Pr}_{\mathrm{sel}}=1/k$, $\mathrm{Pr}_{\mathrm{rep}}(p,q)=1/(e|C|^{2p})$, $\lambda=1$, $\delta=\left(1+\frac{1-\epsilon}{3+2/p}\right)\frac{\epsilon}{k}$, $f_{\mathrm{init}} \leq \sum^n_{i=1} d^{\max}_{i}$, and $\mathrm{OPT} \geq \sum^{n-k}_{i=1} d^{\min}_{(i)}$. Thus, the total expected number of iterations of the GSEMO for achieving the desired approximation ratio is at most 
\begin{align*}
&ek|C| \ln \frac{|C|}{|C|-k}+ek|C|^{2p} \cdot O\left(\frac{k}{\epsilon}\log \frac{\sum^n_{i=1} d^{\max}_{i}}{\sum^{n-k}_{i=1} d^{\min}_{(i)}}\right)\\
&=O\left(\frac{k^2|C|^{2p}}{\epsilon}\log \frac{\sum^n_{i=1} d^{\max}_{i}}{\sum^{n-k}_{i=1} d^{\min}_{(i)}}\right),
\end{align*}
implying that the theorem holds.
\end{proof}

\section{Theoretical Analysis of The GSEMO \\for discrete $k$-median Clustering under Fairness}\label{sec-fairness}

As machine learning has been used increasingly in decision making tasks, the fairness of learning algorithms has become an important research topic. Due to the wide applications, clustering has also been studied from the perspective of fairness, e.g., group fairness~\cite{chierichetti2017fair} which requires all clusters to be balanced with respect to some protected attributes such as gender or race, and individual fairness~\cite{jung2019center} which requires all points to be treated equally, i.e., each point in $\mathcal{D}$ has a center among its $(|\mathcal{D}|/k)$-closest neighbors. 

To examine whether EAs can achieve theoretically guaranteed performance for clustering under fairness, we consider discrete $k$-median clustering under individual fairness~\cite{mahabadi2020individual}, as presented in Definition~\ref{def-median-fair}. The discrete $k$-median clustering problem in Definition~\ref{def-median} is to select a set $X$ of $k$ centers from $F$, to minimize the sum of the distance of each point in $\mathcal{D}$ to the nearest center in $X$. When considering individual fairness, the centers are selected from $\mathcal{D}$ (i.e., $F=\mathcal{D}$), and the selected set $X$ of $k$ centers is required to be $\beta$-fair, where $\beta \geq 1$. Next we introduce the notion of $\beta$-fairness. For any $\bm{v_i} \in \mathcal{D}$, let $$B(\bm{v_i},r)=\{\bm{u} \in \mathcal{D} \mid d(\bm{v_i},\bm{u}) \leq r\}$$ denote the set of points contained by the ball of radius $r$ centered at $\bm{v_i}$, and we use 
\begin{align}\label{eq-median-fair-1}
r(\bm{v_i})=\min\{r \mid |B(\bm{v_i},r)| \geq n/k\}
\end{align} 
to denote the minimum radius such that the ball centered at $\bm{v_i}$ contains at least $n/k$ points from $\mathcal{D}$. Intuitively, $r(\bm{v_i})$ is the radius which $\bm{v_i}$ expects to have a center within, if the $k$ centers are selected uniformly at random from $\mathcal{D}$. A set $X$ of $k$ centers is said to be $\beta$-fair if
$$
\forall \bm{v}_i\in \mathcal{D}: d(\bm{v}_i,X) \leq \beta\cdot r(\bm{v}_i).
$$

\begin{definition}[$\beta$-Fair Discrete $k$-Median Clustering~\cite{mahabadi2020individual}]\label{def-median-fair}
Given a set of $n$ data points $\mathcal{D}=\{\bm{v}_1,\bm{v}_2,\ldots,\bm{v}_n\}$ in $\mathbb{R}^{l}$, a metric distance function $d: \mathcal{D} \times \mathcal{D} \rightarrow \mathbb{R}^+$, an integer $k$, and a parameter $\beta \geq 1$, the goal of $\beta$-fair discrete $k$-median clustering is to find a subset $X \subseteq \mathcal{D}$ of size $k$ such that
\begin{align}\label{eq-median-fair}
& \sum\nolimits_{\bm{v}_i \in \mathcal{D}} d(\bm{v}_i,X)
\end{align}
is minimized under the constraint
\begin{align}\label{eq-median-fair-constraint}
& \forall \bm{v}_i\in \mathcal{D}: d(\bm{v}_i,X) \leq \beta \cdot r(\bm{v}_i), 
\end{align}
where $d(\bm{v}_i,X)=\min \{d(\bm{v}_i,\bm{u}) \mid \bm{u} \in X\}$ is the distance between $\bm{v}_i$ and its closest point in $X$, and $r(\bm{v}_i)$ as in Eq.~(\ref{eq-median-fair-1}) is the minimum radius such that the ball centered at $\bm{v}_i$ contains at least $n/k$ points from $\mathcal{D}$.
\end{definition}

To apply the GSEMO to solve $\beta$-fair discrete $k$-median clustering, a subset $X$ of $\mathcal{D}$ is represented by a Boolean vector $\bm{x} \in \{0,1\}^{n}$, where the $i$-th bit $x_i=1$ iff the $i$-th point in $\mathcal{D}$ is selected as a center, i.e., $\bm{v}_i \in X$. In~\cite{mahabadi2020individual}, it has been proved that if a set $X$ of $k$ centers is feasible with respect to a set $\mathcal{B}$ of critical balls in Definition~\ref{def-critical-ball} (which can be computed in time $O(n^2)$~\cite{mahabadi2020individual}), i.e., $X$ has common points with each critical ball in $\mathcal{B}$, the fairness of $X$ can be guaranteed, which is $7\beta$ as shown in Lemma~\ref{lemma-median-fair-1}.

\begin{definition}[Critical Balls, Definition~2.4 in~\cite{mahabadi2020individual}]\label{def-critical-ball}
A set $\mathcal{B}$ of balls $B(\bm{c}^*_1,\beta r(\bm{c}^*_1)),B(\bm{c}^*_2,\beta r(\bm{c}^*_2))\ldots,B(\bm{c}^*_q,\beta r(\bm{c}^*_q))$ (where $q \leq k$) are called critical if they satisfy
\begin{enumerate}
\item $\forall \bm{v}_i \in \mathcal{D}: d(\bm{v}_i,\{\bm{c}^*_1,\bm{c}^*_2,\ldots,\bm{c}^*_q\}) \leq 6\beta r(\bm{v}_i)$;
\item $\forall i,j\in [q]: d(\bm{c}^*_i,\bm{c}^*_j)>6\beta \max\{r(\bm{c}^*_i),r(\bm{c}^*_j)\}$.
\end{enumerate}
\end{definition}

\begin{lemma}[Lemmas~4.1 and~4.2 in~\cite{mahabadi2020individual}]\label{lemma-median-fair-1}
A set $\mathcal{B}$ of critical balls can be computed in time~$O(n^2)$, and if a set $X$ of $k$ centers is feasible with respect to $\mathcal{B}$, i.e., satisfies
\begin{align}\label{eq-median-fair-3}
\forall B \in \mathcal{B}: |B \cap X| \geq 1,    
\end{align}
then $X$ is $(7\beta)$-fair, i.e., $\forall \bm{v}_i\in \mathcal{D}: d(\bm{v}_i,X) \leq 7\beta \cdot r(\bm{v}_i)$.
\end{lemma}

Inspired by this property, we reformulate the original problem in Definition~\ref{def-median-fair} as a bi-objective maximization problem
\begin{align}\label{eq-bi-median-fair}
&\max\nolimits_{\bm{x} \in \{0,1\}^n} \;\; (f_1(\bm{x}),f_2(\bm{x})), \;\text{where}\\
&\begin{cases}\nonumber
f_1(\bm{x}) = -\sum\limits_{\bm{v}_i \in \mathcal{D}} d(\bm{v}_i,\bm{x})  -\sum\limits^n_{i=1} d^{\max}_{i}\cdot \sum\limits_{B\in \mathcal{B}}\mathbb{I}(B \cap \bm{x} = \emptyset),\\
f_2(\bm x) = |\bm{x}|,
\end{cases}
\end{align}
where $d^{\max}_i=\max \{d(\bm{v}_i,\bm{u}) \mid \bm{u} \in \mathcal{D}\}$, and $\mathbb{I}(\cdot)$ denotes the indicator function which takes 1 if $\cdot$ is true, and 0
otherwise. That is, the GSEMO is to minimize 
\begin{align}\label{eq-median-fair-2}
\sum_{\bm{v}_i \in \mathcal{D}} d(\bm{v}_i,\bm{x})+\sum\limits^n_{i=1} d^{\max}_{i}\cdot \sum\limits_{B\in \mathcal{B}}\mathbb{I}(B \cap \bm{x} = \emptyset)
\end{align} and maximize the subset size $|\bm{x}|$ simultaneously. Note that in Eq.~(\ref{eq-median-fair-2}), the second term $\sum\nolimits^n_{i=1} d^{\max}_{i}\cdot \sum\nolimits_{B\in \mathcal{B}}\mathbb{I}(B \cap \bm{x} = \emptyset)$ enforces a solution with less violation degree with respect to Eq.~(\ref{eq-median-fair-3}) (measured by $\sum_{B\in \mathcal{B}}\mathbb{I}(B \cap \bm{x} = \emptyset)$) to be better; the first term $\sum_{\bm{v}_i \in \mathcal{D}} d(\bm{v}_i,\bm{x})$ enforces a feasible solution with respect to the set $\mathcal{B}$ of critical balls, which has a smaller value of the original objective function, to be better. We set $f_1(\bm{0})$ to $-\infty$. As in the previous three sections, when the GSEMO is applied to solve the reformulated bi-objective problem Eq.~(\ref{eq-bi-median-fair}), the Boolean-vector solutions with size larger than $k$ are excluded; when the GSEMO is terminated, the Boolean-vector solution with size $k$ in the final population will be output.

Lemma~\ref{lemma-median-fair-2} shows that for a feasible set $X$ of $k$ centers with respect to the set $\mathcal{B}$ of critical balls, if there are no feasible swaps of size at most 4 that decrease the objective function value by at least a factor of $1/(1-1/(8k))$, this set $X$ achieves a $84$-approximation ratio. That is, it is always possible to improve a feasible set of $k$ centers by swapping at most $4$ centers, until achieving a good approximation.

\begin{lemma}[Lemmas~5.3,~5.10 and~5.13 in~\cite{mahabadi2020individual}]\label{lemma-median-fair-2}
Let $X$ be a subset of $\mathcal{D}$ with size $k$ (i.e., a set of $k$ centers), which is feasible with respect to the set $\mathcal{B}$ of critical balls. If no feasible subset $X'$ of $G$ with respect to $\mathcal{B}$, having the objective value 
\begin{align*}
\sum\nolimits_{\bm{v}_i \in \mathcal{D}} d(\bm{v}_i,X') \leq \left(1-1/(8k)\right)\cdot \sum\nolimits_{\bm{v}_i \in \mathcal{D}} d(\bm{v}_i,X),
\end{align*}
can be achieved by deleting at most $4$ points inside $X$ and inserting the same number of points outside into $X$, then\begin{align*}
\sum\nolimits_{\bm{v}_i \in \mathcal{D}} d(\bm{v}_i,X) \leq 84\cdot \mathrm{OPT},\end{align*}
where $\mathrm{OPT}$ denotes the optimal value of Eq.~(\ref{eq-median-fair}) under the constraint Eq.~(\ref{eq-median-fair-constraint}), i.e., the objective function value of an optimal $\beta$-fair set of $k$ centers.
\end{lemma}

The above lemma implies that the $\beta$-fair discrete $k$-median clustering problem is $84$-approximated by a $(1-1/(8k))$-approximate local optimum. That is, $\delta=1/(8k)$ in Definition~\ref{def-approx-local}, and $\alpha=84$ in Definition~\ref{def-local-approx}. Furthermore, the parameters $p$ and $q$ in Definition~\ref{def-approx-local} are both equal to 4 here. By applying Lemma~\ref{lemma-general-condition}, we can prove Theorem~\ref{theo-median-fair}, showing that the GSEMO can achieve a $(84,7)$-bicriteria approximation ratio, i.e., the output Boolean-vector solution $\bm{x}$ by the GSEMO satisfies that $\sum_{\bm{v}_i \in \mathcal{D}} d(\bm{v}_i,\bm{x}) \leq 84 \cdot \mathrm{OPT}$, and $\forall \bm{v}_i\in \mathcal{D}: d(\bm{v}_i,\bm{x}) \leq 7\beta \cdot r(\bm{v}_i)$. The proof is similar to that of Theorem~\ref{theo-median}, except that after finding a solution with size $k$ and before achieving the desired approximation ratio, it needs another phase to make the solution feasible with respect to the set $\mathcal{B}$ of critical balls. The required expected number of iterations is polynomial in $n$ and $\log (\max_{i}d^{\max}_{i}/\min_{i}d^{\min}_{i})$.

\begin{theorem}\label{theo-median-fair}
For $\beta$-fair discrete $k$-median clustering in Definition~\ref{def-median-fair}, the expected number of iterations of the GSEMO using Eq.~(\ref{eq-bi-median-fair}), until achieving a $(84,7)$-bicriteria approximation ratio, is at most $O\left(k^2n+k^2n^8\log \frac{\sum^n_{i=k+1} d^{\max}_{(i)}}{\sum^{n-k}_{i=1} d^{\min}_{(i)}}\right)$, where $d^{\max}_{(1)},\ldots,d^{\max}_{(n)}$ and $d^{\min}_{(1)},\ldots,d^{\min}_{(n)}$ are permutations of $d^{\max}_{1},\ldots,d^{\max}_{n}$ and $d^{\min}_{1},\ldots,d^{\min}_{n}$ in ascending order, respectively, $d^{\max}_i=\max \{d(\bm{v}_i,\bm{u}) \mid \bm{u} \in \mathcal{D}\}$, and $d^{\min}_i=\min \{d(\bm{v}_i,\bm{u}) \mid \bm{u} \in \mathcal{D}\setminus \{\bm{v}_i\}\}$.
\end{theorem}
\begin{proof}
We divide the optimization process into three phases: (1)~starts from the initial solution $\bm{0}$ and finishes after finding a solution with size $k$; (2)~starts after phase~(1) and finishes after finding a solution with size $k$ which is feasible with respect to the set $\mathcal{B}$ of critical balls; (3)~starts after phase~(2) and finishes after achieving the desired $(84,7)$-bicriteria approximation ratio. The analysis of phase~(1) can be accomplished as same as that in the proof of Theorem~\ref{theo-median}, except that $|F|=|\mathcal{D}|=n$ here. Thus, the GSEMO needs at most 
\begin{align}\label{eq-median-fair-10}
ekn\ln \frac{n}{n-k}
\end{align}
expected number of iterations to find a solution with size $k$. Note that $k$ is assumed to be smaller than $n$, which obviously holds in practice. 

In phase~(2), let $\hat{\bm{x}}$ denote the solution with size $k$ in the population $P$. The population $P$ will always contain a solution with size $k$ after phase~(1), since it has the largest $f_2$ value $k$ and can be weakly dominated by only other solutions with size $k$. We use $J_{\mathrm{vio}}$ to denote the violation degree of $\hat{\bm{x}}$ with respect to Eq.~(\ref{eq-median-fair-3}), i.e., 
$$J_{\mathrm{vio}}=\sum\nolimits_{B\in \mathcal{B}}\mathbb{I}(B \cap \hat{\bm{x}} = \emptyset).
$$
$J_{\mathrm{vio}}$ is at most $q$, because $\mathcal{B}$ contains $q$ critical balls, where $q\leq k$. When $J_{\mathrm{vio}}=0$, it implies that $\forall B \in \mathcal{B}: B \cap \hat{\bm{x}}\neq \emptyset$, i.e., Eq.~(\ref{eq-median-fair-3}) holds; thus, $\hat{\bm{x}}$ is now feasible with respect to $\mathcal{B}$, i.e., the goal of phase~(2) is reached. It is clear that $J_{\mathrm{vio}}$ cannot increase, because a solution with larger violation degree has a larger value of Eq.~(\ref{eq-median-fair-2}), i.e., a smaller $f_1$ value. Assume that currently $J_{\mathrm{vio}}=i \leq q$. This implies that $\hat{\bm{x}}$ does not intersect with $i$ critical balls. Let $R$ denote the set of points contained by these $i$ critical balls. Adding one point in $R$ into $\hat{\bm{x}}$ will decrease $J_{\mathrm{vio}}$ by 1. By the definition of critical balls in Definition~\ref{def-critical-ball}, we know that the critical balls are disjoint, and each critical ball contains at least $n/k$ points. Thus, $|R|\geq in/k$. Because $\hat{\bm{x}}$ currently intersects with $q-i$ critical balls, there must exist a subset $S$ of $\hat{\bm{x}}$ with $|S|\geq k-(q-i)$, such that deleting one point in $S$ from $\hat{\bm{x}}$ will not increase $J_{\mathrm{vio}}$. Therefore, by selecting the solution $\hat{\bm{x}}$ in line~3 of Algorithm~1; flipping one of the 1-bits corresponding to $S$ (i.e., deleting one point in $S$ from $\hat{\bm{x}}$) and one of the 0-bits corresponding to $R$ (i.e., adding one point in $R$ into $\hat{\bm{x}}$) while keeping the other bits unchanged in line~4 (which performs bit-wise mutation), an offspring solution $\bm{x}'$ with size $k$ will be generated, satisfying $$\sum\nolimits_{B\in \mathcal{B}}\mathbb{I}(B \cap \bm{x}' = \emptyset)=i-1.$$
Compared with $\hat{\bm{x}}$, $\bm{x}'$ has a smaller value of Eq.~(\ref{eq-median-fair-2}), and thus a larger value of $f_1$, implying that $\bm{x}'$ dominates $\hat{\bm{x}}$. Thus, $\bm{x}'$ will be added into the population $P$ and replace $\hat{\bm{x}}$, implying that $J_{\mathrm{vio}}$ is decreased by 1. Then, we analyze the probability of the above selection and mutation behavior. Due to uniform selection, the probability of selecting the solution $\hat{\bm{x}}$ in line~3 of Algorithm~1 is $1/|P| \geq 1/k$, where the inequality holds by the population size $|P| \leq k$ which can be derived as in the proof of Theorem~\ref{theo-median}. The probability of mutation is $(|S|/n) \cdot (|R|/n)\cdot (1-1/n)^{n-2}$, where the first term is the probability of flipping one of the 1-bits corresponding to $S$, the second term is the probability of flipping one of the 0-bits corresponding to $R$, and the last one is the probability of keeping the remaining $n-2$ bits unchanged. Thus, the probability of decreasing $J_{\mathrm{vio}}$ in one iteration of the GSEMO is at least $(1/k)\cdot (|S|/n) \cdot (|R|/n)\cdot (1-1/n)^{n-2}\geq i(k-q+i)/(ek^2n)$, where the inequality is by $|S| \geq k-(q-i)$ and $|R| \geq in/k$. That is, 
$$\mathrm{Pr}[J_{\mathrm{vio}} \;\text{decreases in one iteration} \mid J_{\mathrm{vio}}=i]\geq \frac{i(k-q+i)}{ek^2n}.$$
Because $J_{\mathrm{vio}} \leq q$, the expected number of iterations of phase~(2) (i.e., to make $J_{\mathrm{vio}}=0$) is at most
\begin{align}\label{eq-median-fair-9}
\sum^{q}_{i=1} \frac{ek^2n}{i(k-q+i)} \leq \sum^{k}_{i=1} \frac{ek^2n}{i^2} \leq 2ek^2n,
\end{align}
where the first inequality holds by $q \leq k$, and the last inequality holds by $\sum^{k}_{i=1}1/i^2\leq 1+\sum^k_{i=2}(1/(i-1)-1/i)\leq 2$.

In phase~(3), the population $P$ will always contain a feasible solution of size $k$, with respect to the set $\mathcal{B}$ of critical balls, because a non-feasible solution with respect to $\mathcal{B}$ has a larger value of Eq.~(\ref{eq-median-fair-2}) and thus a smaller value of $f_1$. Note that for a feasible solution $\bm{x}$ of size $k$, $f_1(\bm{x})$ just equals to $-\sum_{\bm{v}_i \in \mathcal{D}} d(\bm{v}_i,\bm{x})$, since $\sum_{B\in \mathcal{B}}\mathbb{I}(B \cap \bm{x} = \emptyset)=0$. The analysis of this phase is similar to that of phase~(2) in the proof of Theorem~\ref{theo-median}. The $\beta$-fair discrete $k$-median clustering problem is $84$-approximated by a $(1-1/(8k))$-approximate local optimum. Similar to the analysis of Eq.~(\ref{eq-median-7}) in the proof of Theorem~\ref{theo-median}, we can use Lemma~\ref{lemma-general-condition} to get that until generating a $(1-1/(8k))$-approximate local optimum, the expected number of iterations is at most
\begin{align}
&\frac{1}{1-(1-\mathrm{Pr}_{\mathrm{sel}} \cdot \mathrm{Pr}_{\mathrm{rep}}(p,q))^{\lambda}} \cdot   O\left(\frac{1}{\delta} \log  \frac{f_{\mathrm{init}}}{\mathrm{OPT}}\right)\nonumber\\
&=ekn^8\cdot  O\left(k\log \frac{\sum^n_{i=k+1} d^{\max}_{(i)}}{\sum^{n-k}_{i=1} d^{\min}_{(i)}}\right),\label{eq-median-fair-8}
\end{align} 
where the equality holds by $\mathrm{Pr}_{\mathrm{sel}}=1/k$, $p=q=4$, $\mathrm{Pr}_{\mathrm{rep}}(p,q)=1/(en^{8})$, $\lambda=1$, $\delta=1/(8k)$, $f_{\mathrm{init}} \leq \sum^n_{i=k+1} d^{\max}_{(i)}$, and $\mathrm{OPT} \geq \sum^{n-k}_{i=1} d^{\min}_{(i)}$. Note that due to $F=\mathcal{D}$ here, we have used a tighter upper bound $\sum^n_{i=k+1} d^{\max}_{(i)}$ for $f_{\mathrm{init}}$, compared with $\sum^n_{i=1} d^{\max}_{i}$ in the proof of Theorem~\ref{theo-median}. For any feasible solution $\bm{x}$ of size $k$, we have
\begin{align*}
\sum_{\bm{v}_i \in \mathcal{D}} d(\bm{v}_i,\bm{x})=\sum_{\bm{v}_i \in \mathcal{D}\setminus \bm{x}} \!\!d(\bm{v}_i,\bm{x})\leq \!\sum_{\bm{v}_i \in \mathcal{D}\setminus \bm{x}} \!\!d^{\max}_{i} \leq \sum^n_{i=k+1} \!d^{\max}_{(i)},
\end{align*}
where the equality holds because $d(\bm{v}_i,\bm{x})=\min\{d(\bm{v}_i,\bm{u})\mid \bm{u} \in \bm{x}\}$ is the distance between $\bm{v_i}$ and its closest point in $\bm{x}$, the first inequality holds by $d(\bm{v}_i,\bm{x})=\min \{d(\bm{v}_i,\bm{u}) \mid \bm{u} \in \bm{x}\} \leq \max \{d(\bm{v}_i,\bm{u}) \mid \bm{u} \in \mathcal{D}\}=d^{\max}_{i}$ due to $\bm{x} \subseteq \mathcal{D}$, and 
the last inequality holds because $|\mathcal{D}\setminus \bm{x}|=n-k$ and $d^{\max}_{(1)},d^{\max}_{(2)},\ldots,d^{\max}_{(n)}$ is a permutation of $d^{\max}_{1},d^{\max}_{2},\ldots,d^{\max}_{n}$ in ascending order. By Lemma~\ref{lemma-median-fair-2}, a $(1-1/(8k))$-approximate local optimum $\bm{x}$ satisfies 
\begin{align}\label{eq-median-fair-6}
\sum\nolimits_{\bm{v}_i \in \mathcal{D}} d(\bm{v}_i,\bm{x}) \leq 84\cdot \mathrm{OPT}.\end{align}
Furthermore, since $\bm{x}$ is feasible with respect to $\mathcal{B}$, we know from Lemma~\ref{lemma-median-fair-1} that 
\begin{align}\label{eq-median-fair-7}
\forall \bm{v}_i\in \mathcal{D}: d(\bm{v}_i,\bm{x}) \leq 7\beta \cdot r(\bm{v}_i).
\end{align}
Eqs.~(\ref{eq-median-fair-6}) and~(\ref{eq-median-fair-7}) imply that the desired $(84,7)$-bicriteria approximation ratio is reached. Thus, Eq.~(\ref{eq-median-fair-8}) gives an upper bound on the expected number of iterations of phase~(3).

By combining the expected number of iterations (i.e., Eqs.~(\ref{eq-median-fair-10}),~(\ref{eq-median-fair-9}) and~(\ref{eq-median-fair-8})) in the above three phases, the GSEMO requires at most
\begin{align*}
&ekn\ln\frac{n}{n-k}+2ek^2n+ekn^8\cdot O\left(k\log \frac{\sum^n_{i=k+1} d^{\max}_{(i)}}{\sum^{n-k}_{i=1} d^{\min}_{(i)}}\right)\\
&=O\left(k^2n+k^2n^8\log \frac{\sum^n_{i=k+1} d^{\max}_{(i)}}{\sum^{n-k}_{i=1} d^{\min}_{(i)}}\right)
\end{align*}
iterations in expectation to achieve a bicriteria approximation ratio of $(84,7)$. Thus, the theorem holds.
\end{proof}

\section{Conclusion}\label{sec-conclusion}

Clustering is an important application of EAs. Previous results are all empirical, while this paper provides theoretical justification for evolutionary clustering by proving the approximation guarantees of the GSEMO (a simple MOEA) for solving four formulations of $k$-clustering, i.e., $k$-\emph{t}MM, $k$-center, discrete $k$-median and $k$-means. We also show that the performance of evolutionary clustering can be theoretically grounded even when considering fairness, by proving the bi-criteria approximation guarantee of the GSEMO for solving discrete $k$-median clustering under individual fairness. Note that we have only derived upper bounds on the approximation ratio. The tightness of these bounds is still open, which is worth studying in the future. 

Though having achieved theoretically guaranteed performance, the GSEMO is indeed very simple, which selects a parent solution from the population uniformly at random, uses bit-wise mutation only to generate an offspring solution, and keeps non-dominated solutions generated so far. Thus, an interesting future work is to study whether better approximation guarantees can be achieved by considering advanced components of EAs, especially noting that some practical MOEAs (e.g., MOEA/D~\cite{li2015primary,huang2021runtime}, NSGA-II~\cite{zheng2021first,bian2022better}, NSGA-III~\cite{doerr2022mathematical}, and SMS-EMOA~\cite{bian2023stochastic}) as well as the effectiveness of some advanced components (e.g., diversity-based parent selection~\cite{osuna2020design}, balanced crossover~\cite{friedrich2022crossover}, heavy-tailed mutation~\cite{doerr2017fast}, and non-elitist survivor selection~\cite{dang2021escaping}) have been theoretically analyzed recently. For example, the heavy-tailed mutation operator introduced by Doerr et al.~\cite{doerr2017fast} does not use the fixed mutation rate $1/n$, but employs the mutation rate $c/n$, where $c$ is chosen randomly according to a heavy-tailed distribution. Such a mutation operator makes the number of flipped bits not strongly concentrated around its mean, and eases having jumps of all sizes in the search space. Thus, it may help the GSEMO jump out of local optima, and lead to better performance.

Note that the theoretical analysis in this paper focuses on the expected number of iterations until achieving a desired approximation ratio. Another perspective is to analyze the expected approximation ratio after running a fixed number of iterations, which is called fixed-budget analysis and very useful for practitioners~\cite{lengler2015fixed,doerr2013method,jansen2020analysing}. Thus, performing fixed-budget analysis for evolutionary clustering will be an interesting future work.

It is also interesting to theoretically analyze evolutionary clustering under more complicated situations, e.g., with outliers~\cite{gupta2017local,bhaskara2019greedy}. As EAs have been successfully applied to solve various machine learning problems~\cite{zhou2019evolutionary}, it is expected to provide theoretical justification for more applications, e.g., evolutionary policy search in reinforcement learning~\cite{such2017deep,yang2023reducing} and evolutionary neural architecture search~\cite{real2017large,lv2022analysis}.

\section*{Acknowledgments}

This work was supported by the National Science Foundation of China (62022039, 62276124). Chao Qian is the corresponding author.

\bibliography{tec-evolutionary-clustering}

\begin{thebibliography}{10}
\providecommand{\url}[1]{#1}
\csname url@samestyle\endcsname
\providecommand{\newblock}{\relax}
\providecommand{\bibinfo}[2]{#2}
\providecommand{\BIBentrySTDinterwordspacing}{\spaceskip=0pt\relax}
\providecommand{\BIBentryALTinterwordstretchfactor}{4}
\providecommand{\BIBentryALTinterwordspacing}{\spaceskip=\fontdimen2\font plus
\BIBentryALTinterwordstretchfactor\fontdimen3\font minus
  \fontdimen4\font\relax}
\providecommand{\BIBforeignlanguage}[2]{{%
\expandafter\ifx\csname l@#1\endcsname\relax
\typeout{** WARNING: IEEEtranS.bst: No hyphenation pattern has been}%
\typeout{** loaded for the language `#1'. Using the pattern for}%
\typeout{** the default language instead.}%
\else
\language=\csname l@#1\endcsname
\fi
#2}}
\providecommand{\BIBdecl}{\relax}
\BIBdecl

\bibitem{ahmadian2019better}
S.~Ahmadian, A.~Norouzi-Fard, O.~Svensson, and J.~Ward, ``Better guarantees for
  $k$-means and {E}uclidean $k$-median by primal-dual algorithms,'' \emph{SIAM
  Journal on Computing}, vol.~49, no.~4, pp. FOCS17--97--FOCS17--156, 2019.

\bibitem{arya2004local}
V.~Arya, N.~Garg, R.~Khandekar, A.~Meyerson, K.~Munagala, and V.~Pandit,
  ``Local search heuristics for $k$-median and facility location problems,''
  \emph{SIAM Journal on Computing}, vol.~33, no.~3, pp. 544--562, 2004.

\bibitem{berkhin2006survey}
P.~Berkhin, ``A survey of clustering data mining techniques,'' in
  \emph{Grouping Multidimensional Data}.\hskip 1em plus 0.5em minus 0.4em\relax
  Springer, 2006, pp. 25--71.

\bibitem{bhaskara2019greedy}
A.~Bhaskara, S.~Vadgama, and H.~Xu, ``Greedy sampling for approximate
  clustering in the presence of outliers,'' in \emph{Advances in Neural
  Information Processing Systems 32}, Vancouver, Canada, 2019, pp.
  11\,146--11\,155.

\bibitem{bian2018tools}
C.~Bian, C.~Qian, and K.~Tang, ``A general approach to running time analysis of
  multi-objective evolutionary algorithms,'' in \emph{Proceedings of the 27th
  International Joint Conference on Artiﬁcial Intelligence}, Stockholm,
  Sweden, 2018, pp. 1405--1411.

\bibitem{bian2022better}
C.~Bian and C.~Qian, ``Better running time of the non-dominated sorting genetic
  algorithm {II (NSGA-II)} by using stochastic tournament selection,'' in
  \emph{Proceedings of the 17th International Conference on Parallel Problem
  Solving from Nature}, Dortmund, Germany, 2022, pp. 428--441.

\bibitem{bian2021robustness}
C.~Bian, C.~Qian, Y.~Yu, and K.~Tang, ``On the robustness of median sampling in
  noisy evolutionary optimization,'' \emph{Science China Information Sciences},
  vol.~64, no.~5, pp. 1--13, 2021.

\bibitem{bian2023stochastic}
C.~Bian, Y.~Zhou, M.~Li, and C.~Qian, ``Stochastic population update can
  provably be helpful in multi-objective evolutionary algorithms,'' in
  \emph{Proceedings of the 32nd International Joint Conference on Artificial
  Intelligence}, Macao, SAR, China.

\bibitem{byrka2017improved}
J.~Byrka, T.~Pensyl, B.~Rybicki, A.~Srinivasan, and K.~Trinh, ``An improved
  approximation for $k$-median and positive correlation in budgeted
  optimization,'' \emph{ACM Transactions on Algorithms}, vol.~13, no.~2, pp.
  1--31, 2017.

\bibitem{chierichetti2017fair}
F.~Chierichetti, R.~Kumar, S.~Lattanzi, and S.~Vassilvitskii, ``Fair clustering
  through fairlets,'' in \emph{Advances in Neural Information Processing
  Systems 30}, Long Beach, CA, 2017, pp. 5029--5037.

\bibitem{dang2021escaping}
D.-C. Dang, A.~Eremeev, and P.~K. Lehre, ``Escaping local optima with
  non-elitist evolutionary algorithms,'' in \emph{Proceedings of the 35th AAAI
  Conference on Artificial Intelligence}, Virtual, 2021, pp. 12\,275--12\,283.

\bibitem{doerr2013method}
B.~Doerr, T.~Jansen, C.~Witt, and C.~Zarges, ``A method to derive fixed budget
  results from expected optimisation times,'' in \emph{Proceedings of the 15th
  ACM Conference on Genetic and Evolutionary Computation}, Amsterdam, The
  Netherlands, 2013, pp. 1581--1588.

\bibitem{doerr2013lower}
B.~Doerr, B.~Kodric, and M.~Voigt, ``Lower bounds for the runtime of a global
  multi-objective evolutionary algorithm,'' in \emph{Proceedings of the 2013
  IEEE Congress on Evolutionary Computation}, Cancun, Mexico, 2013, pp.
  432--439.

\bibitem{doerr2017fast}
B.~Doerr, H.~P. Le, R.~Makhmara, and T.~D. Nguyen, ``Fast genetic algorithms,''
  in \emph{Proceedings of the 19th ACM Conference on Genetic and Evolutionary
  Computation}, Berlin, Germany, 2017, pp. 777--784.

\bibitem{drineas2004clustering}
P.~Drineas, A.~Frieze, R.~Kannan, S.~Vempala, and V.~Vinay, ``Clustering large
  graphs via the singular value decomposition,'' \emph{Machine Learning},
  vol.~56, pp. 9--33, 2004.

\bibitem{droste1998rigorous}
S.~Droste, T.~Jansen, and I.~Wegener, ``A rigorous complexity analysis of the
  (1+1) evolutionary algorithm for separable functions with {B}oolean inputs,''
  \emph{Evolutionary Computation}, vol.~6, no.~2, pp. 185--196, 1998.

\bibitem{friedrich2010approximating}
T.~Friedrich, J.~He, N.~Hebbinghaus, F.~Neumann, and C.~Witt, ``Approximating
  covering problems by randomized search heuristics using multi-objective
  models,'' \emph{Evolutionary Computation}, vol.~18, no.~4, pp. 617--633,
  2010.

\bibitem{friedrich2022crossover}
T.~Friedrich, T.~K{\"o}tzing, A.~Radhakrishnan, L.~Schiller, M.~Schirneck,
  G.~Tennigkeit, and S.~Wietheger, ``Crossover for cardinality constrained
  optimization,'' in \emph{Proceedings of the 24th ACM Conference on Genetic
  and Evolutionary Computation}, Boston, MA, 2022, pp. 1399--1407.

\bibitem{friedrich2015maximizing}
T.~Friedrich and F.~Neumann, ``Maximizing submodular functions under matroid
  constraints by evolutionary algorithms,'' \emph{Evolutionary Computation},
  vol.~23, no.~4, pp. 543--558, 2015.

\bibitem{garza2017improved}
M.~Garza-Fabre, J.~Handl, and J.~Knowles, ``An improved and more scalable
  evolutionary approach to multiobjective clustering,'' \emph{IEEE Transactions
  on Evolutionary Computation}, vol.~22, no.~4, pp. 515--535, 2017.

\bibitem{gong2013complex}
M.~Gong, Q.~Cai, X.~Chen, and L.~Ma, ``Complex network clustering by
  multiobjective discrete particle swarm optimization based on decomposition,''
  \emph{IEEE Transactions on Evolutionary Computation}, vol.~18, no.~1, pp.
  82--97, 2013.

\bibitem{gonzalez1985clustering}
T.~F. Gonzalez, ``Clustering to minimize the maximum intercluster distance,''
  \emph{Theoretical Computer Science}, vol.~38, pp. 293--306, 1985.

\bibitem{guha1999greedy}
S.~Guha and S.~Khuller, ``Greedy strikes back: {I}mproved facility location
  algorithms,'' \emph{Journal of Algorithms}, vol.~31, no.~1, pp. 228--248,
  1999.

\bibitem{gupta2017local}
S.~Gupta, R.~Kumar, K.~Lu, B.~Moseley, and S.~Vassilvitskii, ``Local search
  methods for $k$-means with outliers,'' \emph{Proceedings of the VLDB
  Endowment}, vol.~10, no.~7, pp. 757--768, 2017.

\bibitem{hall1999clustering}
L.~O. Hall, I.~B. Ozyurt, and J.~C. Bezdek, ``Clustering with a genetically
  optimized approach,'' \emph{IEEE Transactions on Evolutionary Computation},
  vol.~3, no.~2, pp. 103--112, 1999.

\bibitem{handl2007evolutionary}
J.~Handl and J.~Knowles, ``An evolutionary approach to multiobjective
  clustering,'' \emph{IEEE Transactions on Evolutionary Computation}, vol.~11,
  no.~1, pp. 56--76, 2007.

\bibitem{har2011geometric}
S.~Har-Peled, \emph{Geometric Approximation Algorithms}.\hskip 1em plus 0.5em
  minus 0.4em\relax Boston, MA: American Mathematical Soc., 2011.

\bibitem{har2004coresets}
S.~Har-Peled and S.~Mazumdar, ``On coresets for $k$-means and $k$-median
  clustering,'' in \emph{Proceedings of the 26th ACM Symposium on Theory of
  Computing}, Chicago, IL, 2004, pp. 291--300.

\bibitem{huang2018bi}
Q.~Huang, X.~Huang, Z.~Kong, X.~Li, and D.~Tao, ``Bi-phase evolutionary
  searching for biclusters in gene expression data,'' \emph{IEEE Transactions
  on Evolutionary Computation}, vol.~23, no.~5, pp. 803--814, 2018.

\bibitem{huang2021runtime}
Z.~Huang, Y.~Zhou, C.~Luo, and Q.~Lin, ``A runtime analysis of typical
  decomposition approaches in {MOEA/D} framework for many-objective
  optimization problems,'' in \emph{Proceedings of the 30th International Joint
  Conference on Artiﬁcial Intelligence}, Virtual, 2021, pp. 1682--1688.

\bibitem{jain1999data}
A.~K. Jain, M.~N. Murty, and P.~J. Flynn, ``Data clustering: A review,''
  \emph{ACM Computing Surveys}, vol.~31, no.~3, pp. 264--323, 1999.

\bibitem{jansen2020analysing}
T.~Jansen, ``Analysing stochastic search heuristics operating on a fixed
  budget,'' \emph{Theory of Evolutionary Computation: Recent Developments in
  Discrete Optimization}, pp. 249--270, 2020.

\bibitem{jung2019center}
C.~Jung, S.~Kannan, and N.~Lutz, ``A center in your neighborhood: {F}airness in
  facility location,'' \emph{arXiv preprint arXiv:1908.09041}, 2019.

\bibitem{kanungoa2004local}
T.~Kanungoa, D.~M. Mountb, N.~S. Netanyahuc, C.~D. Piatkoe, R.~Silvermand, and
  A.~Y. Wuf, ``A local search approximation algorithm for $k$-means
  clustering,'' \emph{Computational Geometry}, vol.~28, pp. 89--112, 2004.

\bibitem{kaufman2009finding}
L.~Kaufman and P.~J. Rousseeuw, \emph{Finding Groups in Data: An Introduction
  to Cluster Analysis}.\hskip 1em plus 0.5em minus 0.4em\relax Hoboken, NJ:
  John Wiley \& Sons, 2009.

\bibitem{kim2011genetic}
J.~Kim, I.~Hwang, Y.-H. Kim, and B.-R. Moon, ``Genetic approaches for graph
  partitioning: {A} survey,'' in \emph{Proceedings of the 13th ACM Conference
  on Genetic and Evolutionary Computation}, Dublin, Ireland, 2011, pp.
  473--480.

\bibitem{knowles2001reducing}
J.~D. Knowles, R.~A. Watson, and D.~W. Corne, ``Reducing local optima in
  single-objective problems by multi-objectivization,'' in \emph{Proceedings of
  the 1st International Conference on Evolutionary Multi-Criterion
  Optimization}, Zurich, Switzerland, 2001, pp. 269--283.

\bibitem{lattanzi2019better}
S.~Lattanzi and C.~Sohler, ``A better $k$-means++ algorithm via local search,''
  in \emph{Proceedings of the 36th International Conference on Machine
  Learning}, Long Beach, CA, 2019, pp. 3662--3671.

\bibitem{Laumanns04}
M.~Laumanns, L.~Thiele, and E.~Zitzler, ``Running time analysis of
  multi-objective evolutionary algorithms on pseudo-{B}oolean functions,''
  \emph{IEEE Transactions on Evolutionary Computation}, vol.~8, no.~2, pp.
  170--182, 2004.

\bibitem{lengler2015fixed}
J.~Lengler and N.~Spooner, ``Fixed budget performance of the (1+1) {EA} on
  linear functions,'' in \emph{Proceedings of the 13th ACM Conference on
  Foundations of Genetic Algorithms}, Aberystwyth, UK, 2015, pp. 52--61.

\bibitem{li2015primary}
Y.~Li, Y.~Zhou, Z.~Zhan, and J.~Zhang, ``A primary theoretical study on
  decomposition-based multiobjective evolutionary algorithms,'' \emph{IEEE
  Transactions on Evolutionary Computation}, vol.~20, no.~4, pp. 563--576,
  2016.

\bibitem{liu2018novel}
W.~Liu, Z.~Wang, X.~Liu, N.~Zeng, and D.~Bell, ``A novel particle swarm
  optimization approach for patient clustering from emergency departments,''
  \emph{IEEE Transactions on Evolutionary Computation}, vol.~23, no.~4, pp.
  632--644, 2018.

\bibitem{lloyd1982least}
S.~Lloyd, ``Least squares quantization in {PCM},'' \emph{IEEE Transactions on
  Information Theory}, vol.~28, no.~2, pp. 129--137, 1982.

\bibitem{lv2022analysis}
Z.~Lv, C.~Qian, G.~G. Yen, and Y.~Sun, ``Analysis of expected hitting time for
  designing evolutionary neural architecture search algorithms,'' \emph{arXiv
  preprint arXiv:2210.05397}, 2022.

\bibitem{mahabadi2020individual}
S.~Mahabadi and A.~Vakilian, ``Individual fairness for $k$-clustering,'' in
  \emph{Proceedings of the 37th International Conference on Machine Learning},
  Virtual Event, 2020, pp. 6586--6596.

\bibitem{matouvsek2000approximate}
J.~Matou{\v{s}}ek, ``On approximate geometric $k$-clustering,'' \emph{Discrete
  \& Computational Geometry}, vol.~24, no.~1, pp. 61--84, 2000.

\bibitem{mukhopadhyay2015survey}
A.~Mukhopadhyay, U.~Maulik, and S.~Bandyopadhyay, ``A survey of multiobjective
  evolutionary clustering,'' \emph{ACM Computing Surveys}, vol.~47, no.~4, pp.
  1--46, 2015.

\bibitem{neumann2006minimum}
F.~Neumann and I.~Wegener, ``Minimum spanning trees made easier via
  multi-objective optimization,'' \emph{Natural Computing}, vol.~5, no.~3, pp.
  305--319, 2006.

\bibitem{neumann2011computing}
F.~Neumann, J.~Reichel, and M.~Skutella, ``Computing minimum cuts by randomized
  search heuristics,'' \emph{Algorithmica}, vol.~59, no.~3, pp. 323--342, 2011.

\bibitem{osuna2020design}
E.~C. Osuna, W.~Gao, F.~Neumann, and D.~Sudholt, ``Design and analysis of
  diversity-based parent selection schemes for speeding up evolutionary
  multi-objective optimisation,'' \emph{Theoretical Computer Science}, vol.
  832, pp. 123--142, 2020.

\bibitem{pourhassan2020runtime}
M.~Pourhassan, V.~Roostapour, and F.~Neumann, ``Runtime analysis of {RLS} and
  (1+1) {EA} for the dynamic weighted vertex cover problem,'' \emph{Theoretical
  Computer Science}, vol. 832, pp. 20--41, 2020.

\bibitem{qian-ppsn16-hyper}
C.~Qian, K.~Tang, and Z.-H. Zhou, ``Selection hyper-heuristics can provably be
  helpful in evolutionary multi-objective optimization,'' in \emph{Proceedings
  of the 14th International Conference on Parallel Problem Solving from
  Nature}, Edinburgh, Scotland, 2016, pp. 835--846.

\bibitem{qian.ijcai15}
C.~Qian, Y.~Yu, and Z.-H. Zhou, ``On constrained {B}oolean {P}areto
  optimization,'' in \emph{Proceedings of the 24th International Joint
  Conference on Artificial Intelligence}, Buenos Aires, Argentina, 2015, pp.
  389--395.

\bibitem{qian2022result}
C.~Qian, D.-X. Liu, and Z.-H. Zhou, ``Result diversification by multi-objective
  evolutionary algorithms with theoretical guarantees,'' \emph{Artificial
  Intelligence}, vol. 309, p. 103737, 2022.

\bibitem{qian2019maximizing}
C.~Qian, Y.~Yu, K.~Tang, X.~Yao, and Z.-H. Zhou, ``Maximizing submodular or
  monotone approximately submodular functions by multi-objective evolutionary
  algorithms,'' \emph{Artificial Intelligence}, vol. 275, pp. 279--294, 2019.

\bibitem{real2017large}
E.~Real, S.~Moore, A.~Selle, S.~Saxena, Y.~L. Suematsu, J.~Tan, Q.~V. Le, and
  A.~Kurakin, ``Large-scale evolution of image classifiers,'' in
  \emph{Proceedings of the 34th International Conference on Machine Learning},
  Sydney, Australia, 2017, pp. 2902--2911.

\bibitem{such2017deep}
F.~P. Such, V.~Madhavan, E.~Conti, J.~Lehman, K.~O. Stanley, and J.~Clune,
  ``Deep neuroevolution: {G}enetic algorithms are a competitive alternative for
  training deep neural networks for reinforcement learning,'' \emph{arXiv
  preprint arXiv:1712.06567}, 2017.

\bibitem{tinos2018nk}
R.~Tin{\'o}s, L.~Zhao, F.~Chicano, and D.~Whitley, ``N{K} hybrid genetic
  algorithm for clustering,'' \emph{IEEE Transactions on Evolutionary
  Computation}, vol.~22, no.~5, pp. 748--761, 2018.

\bibitem{vakilian2022improved}
A.~Vakilian and M.~Yalciner, ``Improved approximation algorithms for
  individually fair clustering,'' in \emph{Proceedings of The 25th
  International Conference on Artificial Intelligence and Statistics}, Virtual
  Event, 2022, pp. 8758--8779.

\bibitem{vassilvitskii2006k}
S.~Vassilvitskii and D.~Arthur, ``$k$-means++: The advantages of careful
  seeding,'' in \emph{Proceedings of the 18th Annual ACM-SIAM Symposium on
  Discrete Algorithms}, New Orleans, LA, 2007, pp. 1027--1035.

\bibitem{doerr2022mathematical}
S.~Wietheger and B.~Doerr, ``A mathematical runtime analysis of the
  non-dominated sorting genetic algorithm {III (NSGA-III)},'' in
  \emph{Proceedings of the 32nd International Joint Conference on Artificial
  Intelligence}, Macao, SAR, China.

\bibitem{wu2008top}
X.~Wu, V.~Kumar, J.~Ross~Quinlan, J.~Ghosh, Q.~Yang, H.~Motoda, G.~J.
  McLachlan, A.~Ng, B.~Liu, P.~S. Yu \emph{et~al.}, ``Top 10 algorithms in data
  mining,'' \emph{Knowledge and Information Systems}, vol.~14, no.~1, pp.
  1--37, 2008.

\bibitem{yang2023reducing}
P.~Yang, L.~Zhang, H.~Liu, and G.~Li, ``Reducing idleness in financial cloud
  via multi-objective evolutionary reinforcement learning based load
  balancer,'' \emph{arXiv preprint arXiv:2305.03463}, 2023.

\bibitem{zheng2021first}
W.~Zheng, Y.~Liu, and B.~Doerr, ``A first mathematical runtime analysis of the
  non-dominated sorting genetic algorithm {II (NSGA-II)},'' in
  \emph{Proceedings of the 36th AAAI Conference on Artiﬁcial Intelligence},
  Virtual, 2022, pp. 10\,408--10\,415.

\bibitem{zhou2019evolutionary}
Z.-H. Zhou, Y.~Yu, and C.~Qian, \emph{Evolutionary Learning: Advances in
  Theories and Algorithms}.\hskip 1em plus 0.5em minus 0.4em\relax Singapore:
  Springer, 2019.

\end{thebibliography}
\bibliographystyle{IEEEtranS}

\begin{IEEEbiography}
[{\includegraphics[width=1in,height=1.25in,clip,keepaspectratio]{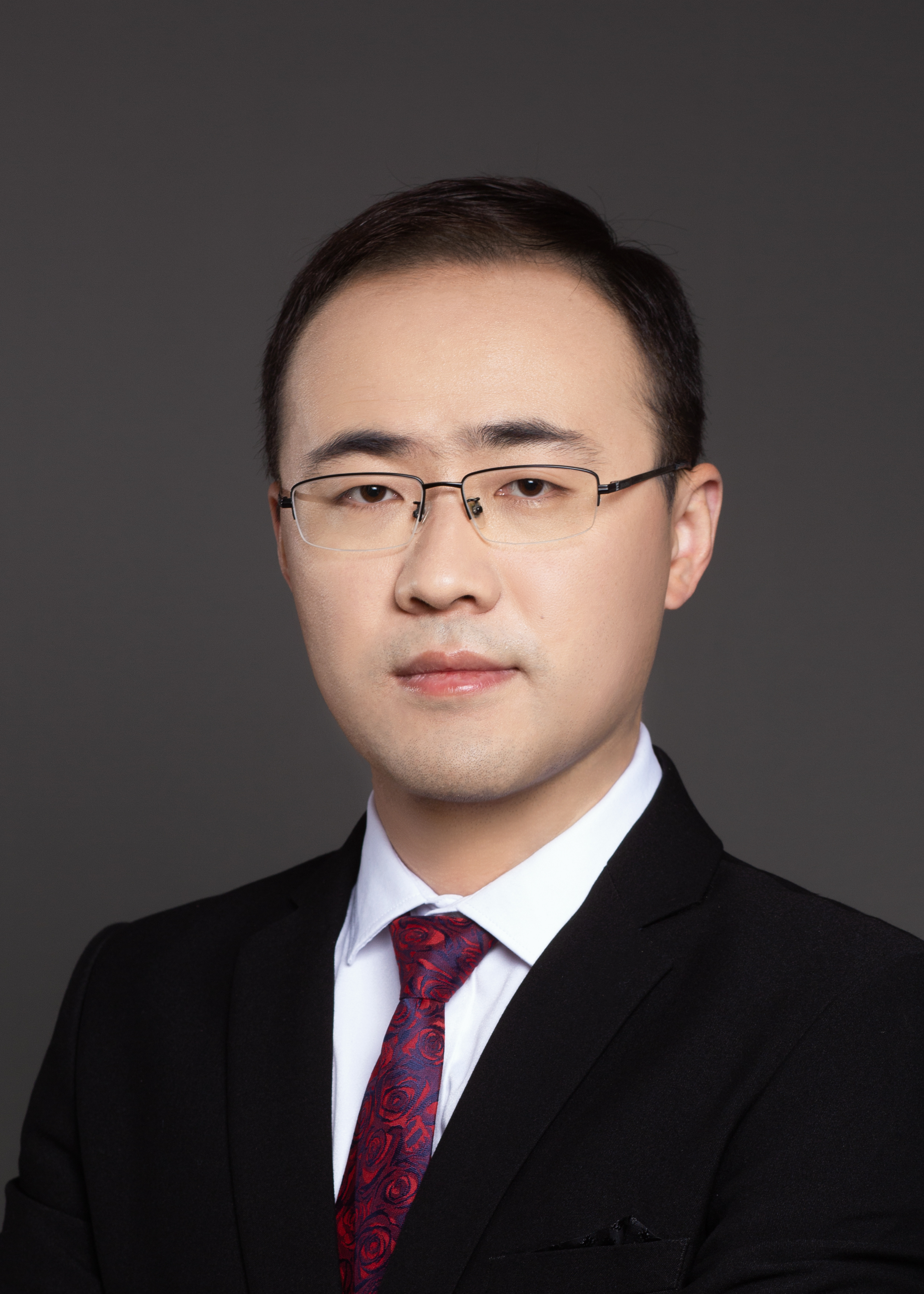}}]{Chao Qian} is an Associate Professor in the School of Artificial Intelligence, Nanjing University, China. He received the BSc and PhD degrees in the Department of Computer Science and Technology from Nanjing University. After finishing his PhD in 2015, he became an Associate Researcher in the School of Computer Science and Technology, University of Science and Technology of China, until 2019, when he returned to Nanjing University.

His research interests are mainly theoretical analysis of evolutionary algorithms (EAs), design of safe and efficient EAs, and evolutionary learning. He has published one book ``Evolutionary Learning: Advances in Theories and Algorithms", and over 40 papers in top-tier journals (AIJ, ECJ, TEvC, Algorithmica, TCS) and conferences (AAAI, IJCAI, NeurIPS, ICLR). He has won the ACM GECCO 2011 Best Theory Paper Award, the IDEAL 2016 Best Paper Award, and the IEEE CEC 2021 Best Student Paper Award Nomination. He is an associate editor of IEEE Transactions on Evolutionary Computation, a young associate editor of Science China Information Sciences, an editorial board member of the Memetic Computing journal, and was a guest editor of Theoretical Computer Science. He is a member of IEEE Computational Intelligence Society (CIS) Evolutionary Computation Technical Committee, and was the chair of IEEE CIS Task Force on Theoretical Foundations of Bio-inspired Computation. He has regularly given tutorials and co-chaired special sessions at leading evolutionary computation conferences (CEC, GECCO, PPSN), and has been invited to give an Early Career Spotlight Talk ``Towards Theoretically Grounded Evolutionary Learning" at IJCAI 2022.
\end{IEEEbiography}

\end{document}